\documentclass{article} %
\usepackage{iclr2027_conference,times}

\usepackage{amsmath,amsfonts,bm}
\usepackage{algorithm,algpseudocode}

\def\eqref#1{equation~\ref{#1}}

\def\1{\bm{1}}

\DeclareMathAlphabet{\mathsfit}{\encodingdefault}{\sfdefault}{m}{sl}
\SetMathAlphabet{\mathsfit}{bold}{\encodingdefault}{\sfdefault}{bx}{n}

\usepackage{amsthm}
\newtheorem{proposition}{Proposition}[section]

\usepackage{graphicx}
\usepackage{hyperref}
\usepackage{url}
\usepackage{etoolbox}
\usepackage{titletoc} %
\usepackage{booktabs}
\usepackage{colortbl}
\usepackage{multirow}
\usepackage{enumitem}
\usepackage{wrapfig}
\usepackage{placeins}
\usepackage{needspace}
\usepackage{capt-of}

\AddToHook{build/column/before}[rewam-fig1-natural-spacing]{%
  \ifnum\value{page}=\getpagerefnumber{fig1}\relax
    \raggedbottom
  \fi
}

\newsavebox{\benchmarktablereference}
\newlength{\benchmarktablewidth}

\newsavebox{\compacttablereference}
\newlength{\compacttablewidth}

\newif\ifrewamsubmissionlayout
\rewamsubmissionlayoutfalse
\ifrewamsubmissionlayout
  \newcommand{\rewamablationplacement}{r}
\else
  \newcommand{\rewamablationplacement}{r}
\fi

\iclrfinalcopy
\makeatletter
\patchcmd{\@maketitle}
  {\lhead{Published as a conference paper at ICLR 2027}}
  {\lhead{}}
  {}{\PackageError{rewam-preprint}{Could not clear the conference-status header}{}}
\makeatother
\renewcommand{\headrulewidth}{0pt}

\title{Reason What Matters: Retrieval-Grounded\\
\resizebox{\textwidth}{!}{Reasoning for Universal Multimodal Embeddings}}

\author{%
\textbf{Mingzhou Jiang}$^{1}$\thanks{Equal contribution. \quad $^{\dagger}$Project lead.} \quad
\textbf{Peixi Wu}$^{2}$\footnotemark[1] \quad
\textbf{Hang Cheng}$^{1}$\footnotemark[1] \quad
\textbf{Yunhao Zhou}$^{3}$\footnotemark[1]\phantom{$^{*}$}$^{\dagger}$ \quad
\textbf{Biao Yang}$^{3}$ \\
\textbf{Wei Yuan}$^{3}$ \quad
\textbf{Yun Li}$^{4}$ \quad
\textbf{Fan Yang}$^{3}$ \quad
\textbf{Wenwu Ou}$^{3}$ \quad
\setcounter{footnote}{2}%
\textbf{Honghui He}$^{1}$\thanks{Corresponding author.} \\
$^{1}$Tsinghua Shenzhen International Graduate School, Tsinghua University \\
$^{2}$School of Artificial Intelligence and Data Science, USTC \\
$^{3}$Kuaishou Technology \\
$^{4}$College of Future Information Technology, Fudan University%
}

\hypersetup{%
  pdftitle={Reason What Matters: Retrieval-Grounded Reasoning for Universal Multimodal Embeddings},
  pdfauthor={Mingzhou Jiang, Peixi Wu, Hang Cheng, Yunhao Zhou, Biao Yang, Wei Yuan, Yun Li, Fan Yang, Wenwu Ou, Honghui He}
}

\begin{document}

\maketitle

\begin{abstract}
Universal multimodal embedding (UME) maps multimodal inputs into a shared
embedding space for diverse retrieval tasks. Recent methods improve embeddings
through Chain-of-Thought (CoT) reasoning optimized with GRPO using retrieval
rewards. However, existing methods overlook the mismatch between candidate-aware
retrieval supervision and input-only CoT generation:
(1) trajectory-level rewards convey retrieval outcomes without explicitly identifying
the input-supported evidence that distinguishes the positive from hard negatives;
(2) input-only generation cannot directly assess whether further reasoning improves retrieval,
potentially producing redundant CoTs with substantial latency.
To bridge this gap, we propose \textbf{Re}ason \textbf{W}h\textbf{a}t
\textbf{M}atters (\textbf{ReWAM}), a retrieval-grounded framework that aligns
candidate-aware supervision with input-only generation. Specifically, we introduce
Retrieval-Aware Self-Distillation (RASD), which extracts privileged guidance from
input-supported facts and evidence distinguishing the positive from hard negatives.
Conditioned on this guidance, an on-policy self-teacher provides token-level
feedback to refine credit assignment, directing policy updates toward
retrieval-relevant reasoning grounded in the input.
We further propose Retrieval-Adaptive Inference (RAI), which learns a
retrieval-aware stopping criterion from prefix-level retrieval feedback.
It stops redundant reasoning without candidate access and uses speculative
decoding to further reduce CoT latency.
Extensive experiments on MMEB-V2 and MRMR demonstrate that ReWAM achieves
state-of-the-art retrieval performance while delivering up to $\bm{5\times}$ the
inference throughput of competitive explicit-CoT UME methods. ReWAM thus enables
high-quality retrieval through efficient input-only reasoning, making explicit CoT
practical for corpus-scale multimodal retrieval. \textit{The code will be publicly available.}
\end{abstract}

\section{Introduction}
\label{sec:introduction}

\begin{figure}[!t]
    \centering
    \includegraphics[width=\linewidth]{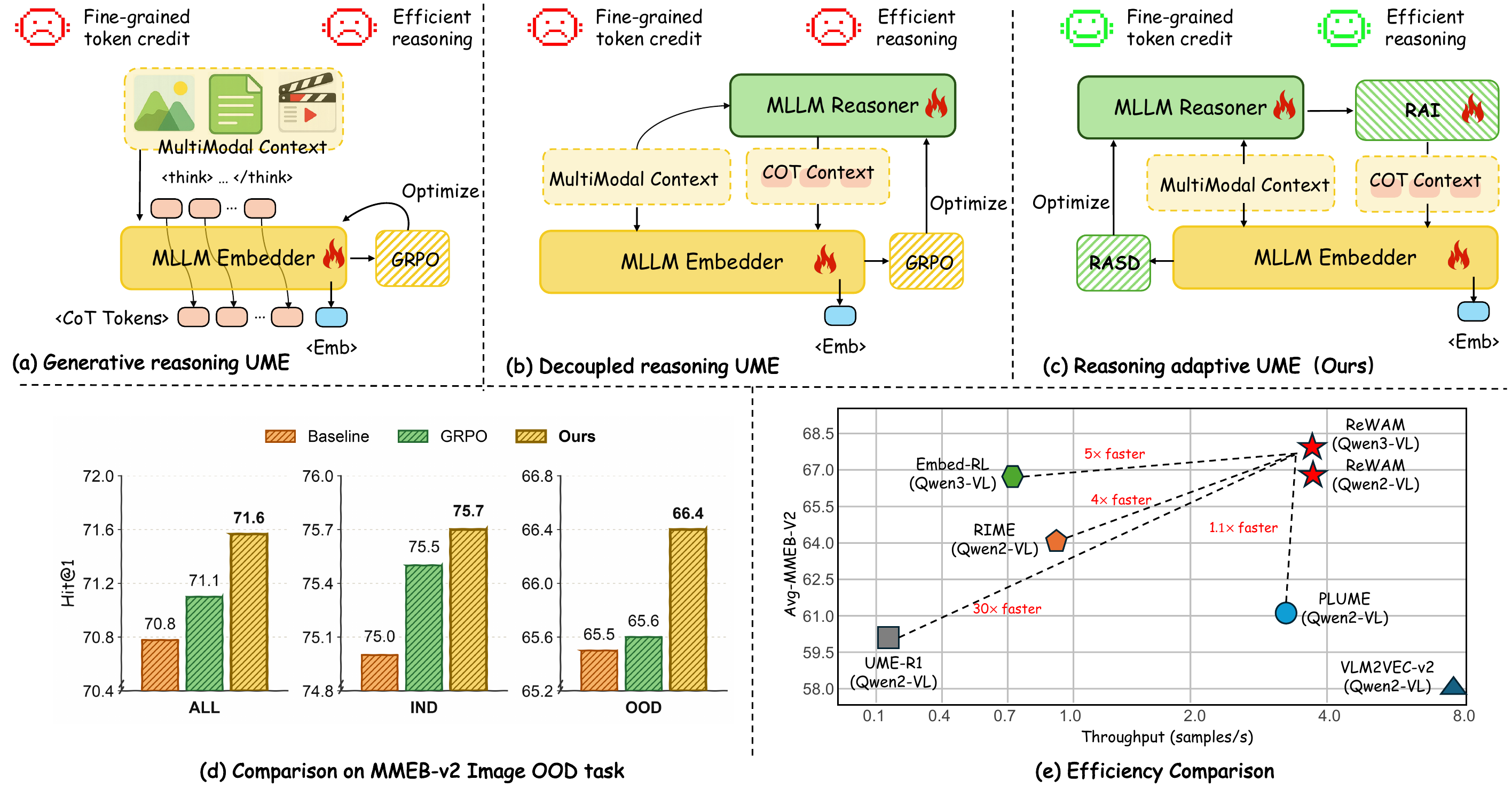}
    \caption{Overview of reasoning-enhanced UME paradigms and empirical
    comparisons. (a)--(c) show Generative reasoning UME, Decoupled reasoning
    UME, and ReWAM, respectively. (d) compares in-distribution and OOD retrieval,
    showing smaller gains from GRPO on OOD tasks. (e)
    compares retrieval performance and throughput across UME methods.}
    \label{fig1}
\end{figure}

Building on vision-language representation learning
\citep{radford2021learning, li2022blip, zhai2023sigmoid}, universal
multimodal embedding (UME) supports heterogeneous retrieval across modalities
and tasks with a single model and a shared embedding space.
Recent methods adapt multimodal large language models (MLLMs) into unified
embedders through contrastive learning
\citep{jiang2024vlm2vec,zhang2025bridging,meng2025vlm2vec}. These models typically
derive embeddings from final-layer hidden states
in a fixed-depth forward pass. Although efficient and scalable, this design
fails to fully leverage the reasoning capabilities of MLLMs, limiting
performance on tasks that require compositional relations, fine-grained
evidence, or multi-step visual understanding.

Recently, some studies have extended UME by incorporating Chain-of-Thought (CoT) reasoning into embedding generation, which we refer to as \textit{Reasoning-enhanced embeddings}. As illustrated in Fig.~\ref{fig1},
prior methods generally fall into two paradigms.
Generative reasoning UME \citep{lan2026ume,wu2026beyond} uses a shared MLLM backbone for both reasoning
and embedding generation.
However, this paradigm often yields suboptimal results due to potential gradient
conflicts between the next-token prediction and contrastive learning objectives
\citep{cui2026think}. Decoupled reasoning UME \citep{cui2026think,jiang2026embed,zhang2026think} separates the reasoning and
embedding components, conditioning the embedder on both the original input and
the generated CoT.
Despite the extra parameters and computation, decoupling
mitigates conflicts between training objectives and often leads to better
retrieval performance
\citep{jiang2026embed}.
Recent methods in both paradigms further optimize CoT generation with Group
Relative Policy Optimization (GRPO) \citep{shao2024deepseekmath}, using retrieval
rewards to encourage retrieval-relevant reasoning.

Despite these advances, existing methods generate CoTs from the input alone to
support independent encoding and reuse of corpus representations, but assess
their retrieval value through comparisons with positive and negative candidates.
Input-supported facts may not distinguish confusable candidates, while seemingly
discriminative claims may lack input support. Existing methods overlook this mismatch between
candidate-aware supervision and input-only reasoning:
(1) Trajectory-level rewards reflect the retrieval outcome of a complete CoT
without identifying which claims are both input-supported and discriminative.
Applying the same advantage to all tokens can therefore reinforce incorrect or
irrelevant claims alongside useful reasoning.
As shown in Fig.~\ref{fig1}(d), GRPO improves in-distribution retrieval but
provides only marginal gains on out-of-distribution (OOD) tasks.
(2) A shorter COT prefix may already contain sufficient retrieval evidence, but
independent generation cannot directly compare candidates to assess the benefit
of continuing. As a result, the model may continue after obtaining sufficient evidence,
adding redundant computation and latency.

To address these limitations, we propose \textbf{Re}ason \textbf{W}h\textbf{a}t \textbf{M}atters (\textbf{ReWAM}), a retrieval-grounded framework that
aligns candidate-aware retrieval supervision with input-only reasoning,
as illustrated in Fig.~\ref{fig1}(c). Specifically, we introduce Retrieval-Aware
Self-Distillation (RASD), which uses a multimodal analyzer to extract
input-supported facts and distinctions between the positive and retrieved hard
negatives. The analyzer uses this evidence to identify supported and conflicting
CoT claims, constructing privileged guidance for an on-policy
self-teacher. Conditioned on this guidance, the teacher provides token-specific
feedback to modulate trajectory advantages and focus policy updates on
input-grounded reasoning that distinguishes confusable candidates.
We further introduce Retrieval-Adaptive Inference (RAI), which learns a
retrieval-aware stopping criterion from prefix-level candidate comparisons for
input-only generation. RAI estimates remaining retrieval utility from the current
reasoning state alone to truncate redundant reasoning, while speculative decoding
reduces generation cost without sacrificing the retrieval benefits of explicit CoT.

Extensive experiments on MMEB-V2 and MRMR demonstrate that ReWAM achieves
state-of-the-art retrieval performance with substantially improved inference
efficiency. As shown in Fig.~\ref{fig1}(e), ReWAM outperforms the decoupled
explicit-CoT method Embed-RL \citep{jiang2026embed} while delivering up to
$5\times$ its throughput. It even surpasses the efficiency-oriented
latent-reasoning method PLUME \citep{he2026plume} in both retrieval quality and
throughput while retaining explicit CoT. These results show that aligning
candidate-aware supervision with input-only reasoning enables high-quality
retrieval and efficient encoding at corpus scale.
Our main contributions are:
\begin{itemize}[leftmargin=*]
    \item We propose ReWAM to align candidate-aware retrieval supervision with
    input-only CoT generation. This guides reasoning toward input-supported
    evidence that distinguishes relevant targets, improving CoT-based retrieval
    while preserving independent query and target encoding.
    \item We introduce RASD for token-level credit assignment through retrieval
    evidence verification, and RAI for adaptive CoT truncation guided by
    prefix-level remaining retrieval utility.
    \item Extensive experiments on MMEB-V2 and MRMR demonstrate state-of-the-art
    retrieval performance with up to $5\times$ the throughput of prior decoupled
    explicit-CoT methods.
\end{itemize}

\section{Related Work}
\label{sec:related-work}

\noindent\textbf{Universal Multimodal Embedding.}\enspace
Contrastive vision--language pretraining learns to align representations across
images and text
\citep{radford2021learning, li2022blip, zhai2023sigmoid}. Building on this
foundation, UME methods adapt MLLMs into general-purpose encoders and explore
diverse architectures and training objectives for multimodal retrieval
\citep{jiang2024vlm2vec,zhang2025bridging,meng2025vlm2vec,faysse2025colpali,
liu2025lamra,chen2026umer,yu2025cafe}. More recent studies integrate reasoning
into representation learning to better capture compositional relations and
fine-grained cross-modal cues
\citep{lan2026ume,wu2026beyond,cui2026think,
jiang2026embed,he2026plume,tsai2026let}.

\noindent\textbf{On-Policy Self-Distillation.}\enspace
Outcome-based RL provides verifiable feedback, yet sparse trajectory rewards do
not reveal the contribution of intermediate decisions
\citep{shao2024deepseekmath}.
On-policy self-distillation addresses this gap by conditioning the policy on
privileged information to supervise its own rollouts at the token level
\citep{zhao2026self,hubotter2026reinforcement,yang2026self,song2026survey}. Later methods refine these
signals through confidence weighting, cross-rollout reflection, skill guidance,
and self-evolving supervision
\citep{zheng2026scope,zheng2026group,yang2026opid,lu2026self,wu2026seed}.

\noindent\textbf{Inference Acceleration.}\enspace
Autoregressive decoding generates tokens sequentially, making long reasoning
traces costly at inference time. Speculative decoding reduces this cost by
verifying multiple draft tokens in parallel \citep{leviathan2023fast}, while
multi-token prediction and multi-head decoding predict several future tokens per
step \citep{gloeckle2024better,cai2024medusa}. Block-diffusion and
semi-autoregressive methods further increase decoding parallelism
\citep{chen2026dflash,cheng2026dspark}. Complementary methods reduce reasoning
length through efficient input routing
\citep{zhang2026think} or task-agnostic semantic
redundancy detection \citep{sun2026stop}.

\section{Method}
\label{sec:method}

\begin{figure}[!t]
    \centering
    \includegraphics[width=\linewidth]{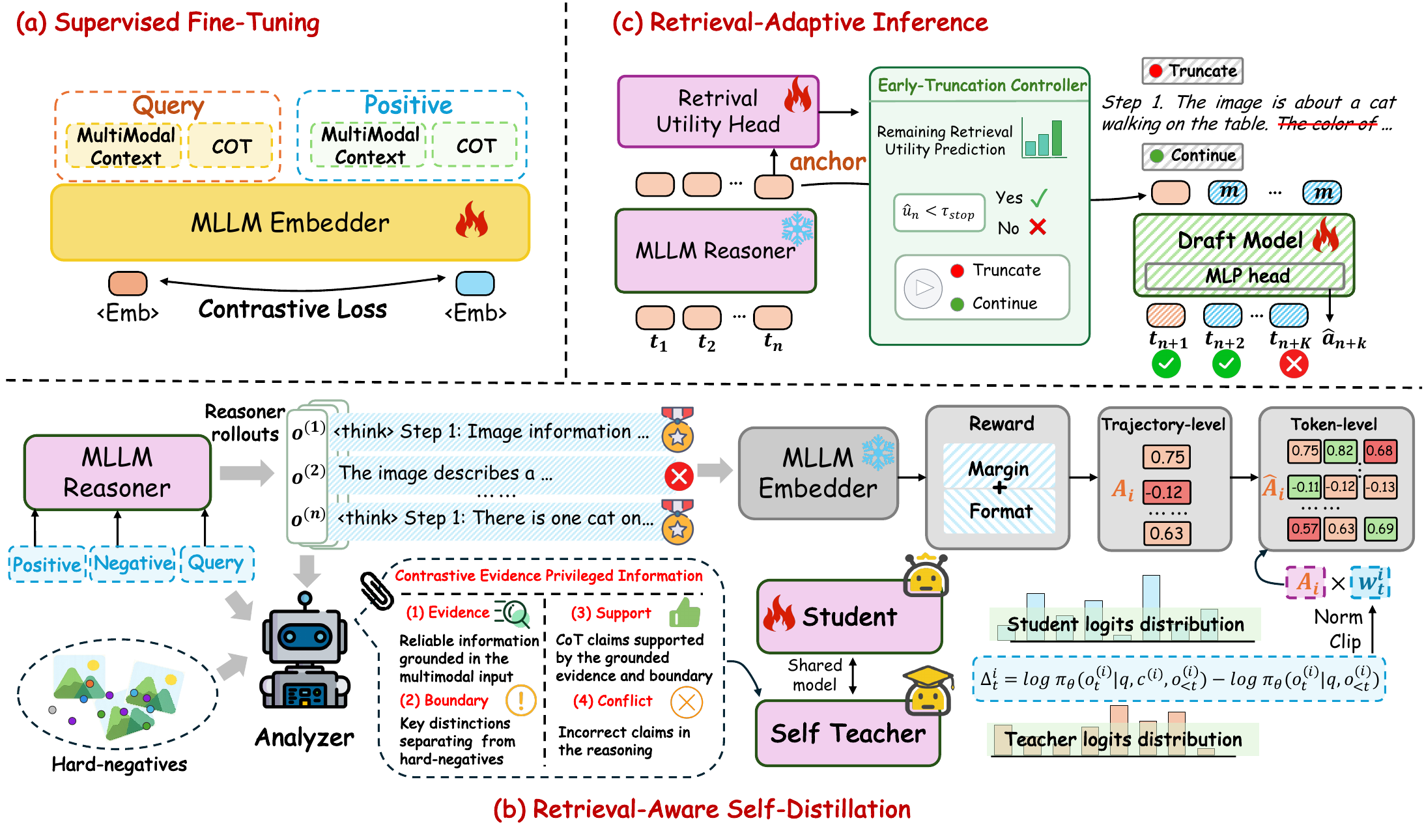}
    \caption{Overview of ReWAM. (a) Supervised fine-tuning trains a
    CoT-conditioned embedder. (b) RASD uses privileged same-token rescoring to
    derive token-wise weights that modulate trajectory-level retrieval advantages. (c) RAI stops a
    partial CoT when its predicted remaining retrieval utility is low and
    accelerates generation through speculative decoding.}
    \label{fig2}
\end{figure}

\subsection{Preliminaries}

Universal multimodal retrieval aims to identify relevant targets for queries
across text, images, videos, and visual documents.
To learn discriminative multimodal embeddings, we adopt contrastive learning
with the InfoNCE loss.
For a batch of $B$ pairs, we optimize the model to maximize similarity
between each query $q_i$ and its positive target $t_i^+$ while minimizing
similarities to all negative targets in $\mathcal T^-=\{t_j^-\mid j\neq i\}$:
\begin{equation}
    \mathcal L_{\mathrm{InfoNCE}}=-\frac{1}{B}\sum_{i=1}^{B}
    \log\frac{\exp(\operatorname{cos}(\mathbf e_{q_i},\mathbf e_{t_i^+})/\tau_E)}
    {\exp(\operatorname{cos}(\mathbf e_{q_i},\mathbf e_{t_i^+})/\tau_E)
    +\sum_{t_j^-\in\mathcal T^-}\exp(\operatorname{cos}(\mathbf e_{q_i},\mathbf e_{t_j^-})/\tau_E)}.
    \label{eq:contrastive-prelim}
\end{equation}
Here $\mathbf e_{q_i}$ and $\mathbf e_t$ denote $\ell_2$-normalized query
and target embeddings, extracted from the final-layer hidden states of the
MLLM embedder conditioned on the multimodal input and its corresponding CoT.
$\operatorname{cos}(\cdot,\cdot)$ denotes cosine similarity, and $\tau_E$
is the temperature.
We compute the InfoNCE loss in both query-to-target and target-to-query
directions and sum the two losses to obtain $\mathcal L_{\mathrm{con}}$.

\subsection{Supervised Fine-Tuning}

As shown in Fig.~\ref{fig2}(a), the embedder $E_\phi$ encodes each query and target
together with an offline CoT and is fine-tuned using the bidirectional
contrastive objective $\mathcal L_{\mathrm{con}}$. No language-modeling loss is applied,
keeping representation learning separate from autoregressive generation. We
then freeze $E_\phi$ to evaluate and optimize reasoner rollouts with a stable
retrieval criterion.

\subsection{Retrieval-Aware Self-Distillation}
\label{sec:rasd}

To address coarse credit assignment in retrieval-based RL, RASD constructs
privileged guidance by contrasting the positive with hard negatives
(Fig.~\ref{fig2}(b)). An on-policy self-teacher uses this guidance to derive
token-level feedback for reweighting trajectory advantages. The resulting
supervision focuses policy updates on reasoning that distinguishes the target
from confusable candidates.

\noindent\textbf{Retrieval-based reward function.}\enspace
Specifically, for each query $q$, the reasoner policy $\pi_\theta$
samples a group of $G$ CoTs $\{o^{(i)}\}_{i=1}^{G}$.
We apply RASD symmetrically to targets, with paired queries as positives
and other in-batch queries as negatives.
Absolute retrieval performance alone is insufficient to assess the contribution
of CoT, as embeddings obtained without CoT may already be highly discriminative.
To quantify this contribution, we define a Margin Reward that estimates
the margin improvement over a no-CoT embedding baseline on the same candidates.
For the positive target $t^+$ and each in-batch negative
$t^-\in\mathcal T^-$, the reasoner likewise samples $G$ CoTs independently.
The frozen embedder produces the $\ell_2$-normalized query embedding
$\mathbf e_q^i=E_\phi(q,o^{(i)})$ and the corresponding positive and
negative target embeddings $\mathbf e_{t^+}^j$ and $\mathbf e_{t^-}^j$.
The Margin Reward compares $M(o^{(i)})$ with $M(\emptyset)$,
the margin obtained with the embedder conditioned on multimodal inputs alone:
\begin{equation}
\begin{gathered}
    M(o^{(i)})=\mathbb E_j\!\left[
    \operatorname{cos}(\mathbf e_q^i,\mathbf e_{t^+}^j)\right]
    -\mathbb E_{t^-,j}\!\left[
    \operatorname{cos}(\mathbf e_q^i,\mathbf e_{t^-}^j)\right],\\[4pt]
    R_{\mathrm{margin}}(o^{(i)})=
    \begin{cases}
        +1, & M(o^{(i)})-M(\emptyset)>\epsilon_v,\\
        0, & |M(o^{(i)})-M(\emptyset)|\leq\epsilon_v,\\
        -1, & M(o^{(i)})-M(\emptyset)<-\epsilon_v,
    \end{cases}
\end{gathered}
    \label{eq:retrieval-margin}
\end{equation}
Here $\epsilon_v=0.01$ suppresses minor embedding fluctuations.
Using the no-CoT margin as an anchor, this reward encourages more discriminative reasoning.
The Format reward $R_{\mathrm{format}}(o^{(i)})$ is $1$ for CoT enclosed in
\texttt{<think>} and \texttt{</think>}, and $0$ otherwise.
The trajectory-level advantage is:
\begin{equation}
    r_i=R_{\mathrm{margin}}(o^{(i)})+\lambda_f R_{\mathrm{format}}(o^{(i)}),
    \qquad A_i=\frac{r_i-\operatorname{mean}(\{r_1,\ldots,r_G\})}
    {\operatorname{std}(\{r_1,\ldots,r_G\})}.
    \label{eq:trajectory-advantage}
\end{equation}

\noindent\textbf{Privileged guidance construction.}\enspace
The reasoner generates each CoT from the input alone, whereas $A_i$ summarizes
its retrieval feedback from candidate comparisons. This trajectory-level signal
does not identify which claims are supported by the input and distinguish the
positive from hard negatives. Applying $A_i$ to all tokens can therefore reinforce
unsupported or irrelevant claims alongside useful reasoning.
To provide this missing guidance, RASD uses a multimodal analyzer $\mathcal A$
to construct Contrastive Evidence Privileged Information (CEPI):
\begin{equation}
    c^{(i)}=\mathcal A(q,t^+,h^-,o^{(i)})
    =\bigl(c^{\mathrm{evi}},c^{\mathrm{bnd}},
    c^{\mathrm{sup}},c^{\mathrm{con}}\bigr),
    \label{eq:cepi-construction}
\end{equation}
where $h^-\subseteq\mathcal T^-$ contains the \textit{top-3} in-batch hard negatives
retrieved by the frozen embedder without CoT.
The analyzer extracts input-grounded facts $c^{\mathrm{evi}}$ and
positive--negative decision boundaries $c^{\mathrm{bnd}}$, then checks
$o^{(i)}$ against both to identify supported claims $c^{\mathrm{sup}}$ and
conflicting claims $c^{\mathrm{con}}$.

\noindent\textbf{Token-level credit assignment.}\enspace
Conditioned on $c^{(i)}$, the same reasoner acts as an
on-policy self-teacher and rescores $o^{(i)}$ without generating a new CoT.
For the $t$-th token $o_t^{(i)}$ with prefix $o_{<t}^{(i)}$,
we estimate teacher--student divergence using their log-probability difference:
\begin{equation}
    \Delta_t^i=\log\pi_\theta(o_t^{(i)}\mid q,c^{(i)},o_{<t}^{(i)})-
    \log\pi_\theta(o_t^{(i)}\mid q,o_{<t}^{(i)}),
    \quad
    s_t^i=\exp\!\left(\operatorname{sign}(A_i)\Delta_t^i\right).
    \label{eq:privileged-delta}
\end{equation}
The likelihood difference $\Delta_t^i$ measures how privileged evidence changes
the policy's support for each sampled token, while $s_t^i$ exponentiates this
feedback after alignment with the advantage sign. Treating $\Delta_t^i$ as a
stop-gradient signal for token-level credit assignment, RASD derives bounded
weights $w_t^i$ to modulate the trajectory advantage $A_i$:
\begin{equation}
    w_t^i=1+\lambda_R\operatorname{clip}\!\left(
    s_t^i-1,
    -\epsilon_R,\epsilon_R\right),
    \qquad
    \hat A_t^i=A_iw_t^i,
    \label{eq:rasd-weight}
\end{equation}
where $\lambda_R$ controls the strength of token-level modulation,
and $\epsilon_R$ is the clipping threshold that bounds the adjustment.
The modulated advantage $\hat A_t^i$ strengthens reinforcement of
evidence-supported tokens when $A_i>0$ and reduces their penalties when
$A_i<0$. The positive weights preserve the sign of each trajectory's advantage.
The RASD objective is:
\begin{equation}
\begin{aligned}
    \mathcal L_{\mathrm{RASD}}
    &=-\mathbb E_{i,t}\!\left[\min\!\left(
    \rho_t^i \hat A_t^i,
    \operatorname{clip}(\rho_t^i,1-\epsilon_P,1+\epsilon_P)
    \hat A_t^i\right)\right]
    +\beta\,D_{\mathrm{KL}}(\pi_\theta\Vert\pi_{\mathrm{ref}}),
\end{aligned}
    \label{eq:rasd-objective}
\end{equation}
where $\rho_t^i=\pi_\theta(o_t^{(i)}\mid q,o_{<t}^{(i)})/
\pi_{\mathrm{old}}(o_t^{(i)}\mid q,o_{<t}^{(i)})$ is the per-token importance
sampling ratio. $\epsilon_P$ is the clipping threshold.
The KL term $D_{\mathrm{KL}}$ penalizes deviation from the reference policy $\pi_{\mathrm{ref}}$.
This objective encourages reasoning grounded in input evidence that distinguishes confusable candidates.

\subsection{Retrieval-Adaptive Inference}
\label{sec:rai}

A CoT prefix may already contain sufficient retrieval evidence, yet input-only
generation cannot directly compare candidates to assess the value of continuing.
This can lead to redundant reasoning and increased latency.
To reduce this overhead, we introduce Retrieval-Adaptive Inference (RAI),
combining a Retrieval Utility Head trained on prefix-level retrieval feedback
for early stopping with speculative decoding for faster generation
(Fig.~\ref{fig2}(c)).

\noindent\textbf{Retrieval-aware early stopping.}\enspace
To learn a retrieval-grounded stopping criterion, we quantify the retrieval
gains still available beyond each CoT prefix. For each query $q$, the frozen reasoner performs one greedy rollout to obtain
a complete CoT $o$ of length $L$.
Each prefix $o_{\leq n}$ ($n\leq L$) contains the first $n$ tokens of $o$
and is encoded as $\mathbf e_q^n=E_\phi(q,o_{\leq n})$.
The candidate embeddings $\mathbf e_{t^+}$ and
$\mathbf e_{t^-}$ remain fixed and are conditioned on their own full CoTs.
The no-CoT margin $m_0$ uses $E_\phi(q)$ against these same candidates.
The prefix margin and remaining retrieval utility are:
\begin{equation}
\begin{gathered}
   m_n=\operatorname{cos}(\mathbf e_q^n,\mathbf e_{t^+})
   -\mathbb E_{t^-\in\mathcal T^-}\!\left[
   \operatorname{cos}(\mathbf e_q^n,\mathbf e_{t^-})\right],\\[3pt]
   g_n=\max\{m_n-m_0,0\},\qquad
   u_n=\frac{\max\{g_k\mid n\leq k\leq L\}-g_n}
   {\max\{g_k\mid 1\leq k\leq L\}}.
\end{gathered}
   \label{eq:remaining-utility}
\end{equation}
With $k$ indexing evaluated prefix lengths, the numerator measures the
largest additional retrieval gain available beyond the current prefix,
while the denominator is the maximum gain over the no-CoT baseline across
the rollout. Their ratio $u_n$ quantifies the relative benefit of continued
reasoning, with values near zero indicating little further gain.
To predict this utility during generation, we use a Retrieval Utility Head
$H_\omega$, a two-hidden-layer MLP with sigmoid output, that maps the frozen
reasoner's final-layer state $h_n$ after the $n$-th token to $\hat u_n$.
We train the head using the offline utility labels $u_n$ with a SmoothL1 loss:
\begin{equation}
    \hat u_n=H_\omega(h_n),\qquad
    \mathcal L_H=\mathbb E_{(h_n,u_n)}\!\left[
    \operatorname{SmoothL1}(\hat u_n,u_n)\right].
    \label{eq:halting-loss}
\end{equation}
At inference, generation stops when $\hat u_n<\tau_{\mathrm{stop}}$,
using the current hidden state without candidate comparisons or embedder calls.
See Appendix~\ref{app:rai-implementation} for prefix sampling and decoding details.

\noindent\textbf{Speculative decoding.}\enspace
To accelerate CoT generation, we adopt speculative decoding with a
semi-autoregressive Draft model \citep{chen2026dflash,cheng2026dspark},
using the frozen reasoner as the target model to verify draft tokens.
As shown in Fig.~\ref{fig2}(c), conditioned on multimodal context features $F_n$
extracted from the frozen reasoner, the Draft backbone $D$ processes the last prefix token
$t_n$ and $K-1$ mask tokens $m$ in parallel.
A low-rank Markov head adds a correction based on the preceding token:
\begin{equation}
\begin{gathered}
    (z_{n+1},\ldots,z_{n+K})=D([t_n,m,\ldots,m];F_n),\\
    q_{n+k}=\operatorname{softmax}\!\left(
    W_{\mathrm{LM}}z_{n+k}+W_2W_1(t_{n+k-1})\right),
    \quad k=1,\ldots,K.
\end{gathered}
    \label{eq:draft-generation}
\end{equation}
Here $q_{n+k}$ is the Draft model's vocabulary distribution for $t_{n+k}$,
and $z_{n+k}$ is its prediction state, mapped to base logits by the frozen
$W_{\mathrm{LM}}$. Using the preceding token $t_{n+k-1}$,
$W_2W_1$ provides a logit correction for predicting the next token $t_{n+k}$.
The backbone computes all $K$ hidden states in parallel, while the lightweight
Markov head generates tokens sequentially.

To learn proposals consistent with the frozen reasoner, we train on offline
CoTs using reference tokens $t^\star_{n+k}$ and target distributions $p_{n+k}$
under the same prefix. Following \citet{cheng2026dspark}, we combine next token
prediction, distribution alignment, and acceptance estimation:
\begin{equation}
    \mathcal L_D=\mathbb E_{n,k}\Bigl[\omega_k\Bigl(
    -\lambda_{\mathrm{ntp}}\log q_{n+k}(t^\star_{n+k})
    +\lambda_{\mathrm{align}}\lVert q_{n+k}-p_{n+k}\rVert_1
    +\lambda_{\mathrm{acc}}\ell^{\mathrm{acc}}_{n+k}\Bigr)\Bigr].
    \label{eq:draft-objective}
\end{equation}
The $\lambda$ coefficients balance the losses, and unit-mean weights
$\omega_k\propto e^{-k/\gamma}$ emphasize earlier proposals over valid training
positions, with $\gamma$ controlling the decay.
The next-token prediction loss supervises reference continuations, while distribution alignment
promotes overlap and hence acceptance under speculative sampling.
The auxiliary acceptance loss is defined as:
\begin{equation}
    a_{n+k}=1-\tfrac12\lVert q_{n+k}-p_{n+k}\rVert_1,\quad
    \ell^{\mathrm{acc}}_{n+k}
    =\operatorname{BCE}(\hat a_{n+k},a_{n+k}).
    \label{eq:draft-acceptance-loss}
\end{equation}
Here $a_{n+k}$ is the Draft--target distribution overlap, equal to the expected
acceptance probability under speculative sampling. An auxiliary sigmoid head
predicts $\hat a_{n+k}$ from $z_{n+k}$ and the preceding token.
BCE denotes binary cross-entropy, which trains the head to match the soft
target $a_{n+k}$.

\suppressfloats[t]
\setcounter{table}{0}
\begin{table}[t]
    \centering
    \caption{Results on the MMEB-V2 benchmark. Best and second-best
    scores are bolded and underlined, respectively. CLS: classification, QA: question answering,
    RET: retrieval, GD: grounding, MRET: moment retrieval, VDR: ViDoRe, VR:
    VisRAG, and OOD: out-of-distribution. We follow the evaluation protocol of VLM2Vec-V2.}
    \label{tab:mmeb-v2-results}
    \vspace{\baselineskip}
    \sbox{\benchmarktablereference}{%
    \begin{tabular}{lcccccccccccccccc}
        \toprule
        \multirow{2}{*}{Model}
        & \multicolumn{5}{c}{Image} & \multicolumn{5}{c}{Video}
        & \multicolumn{5}{c}{VisDoc} & \multirow{2}{*}{All} \\
        \cmidrule(lr){2-6}\cmidrule(lr){7-11}\cmidrule(lr){12-16}
        & CLS & QA & RET & GD & Avg.
          & CLS & QA & RET & MRET & Avg.
          & VDRv1 & VDRv2 & VR & OOD & Avg. & \\
        \midrule
        \# of Datasets & 10 & 10 & 12 & 4 & 36 & 5 & 5 & 5 & 3 & 18
          & 10 & 4 & 6 & 4 & 24 & 78 \\
        \midrule
        \multicolumn{17}{c}{\textit{$\sim$2B model size}} \\
        \midrule
        ColPali-V1.3 (PaliGemma-3B)
          & 40.3 & 11.5 & 48.1 & 40.3 & 34.9 & 26.7 & 37.8 & 21.6 & 25.5 & 28.2
          & \underline{83.6} & 52.0 & 81.1 & 43.1 & 71.0 & 44.4 \\
        GME (Qwen2-VL-2B)
          & 54.4 & 29.9 & 66.9 & 55.5 & 51.9 & 34.9 & 42.0 & 25.6 & 32.4 & 33.9
          & \textbf{86.1} & \underline{54.0} & 82.5 & 43.1 & 72.7 & 54.1 \\
        VLM2Vec (Qwen2-VL-2B)
          & 58.7 & 49.3 & 65.0 & 72.9 & 59.7 & 33.4 & 30.5 & 20.6 & 33.0 & 29.0
          & 49.8 & 13.5 & 51.8 & 33.5 & 41.6 & 47.0 \\
        VLM2Vec-V2 (Qwen2-VL-2B)
          & 62.9 & 56.3 & 69.5 & 77.3 & 64.9 & 39.3 & 34.3 & 28.8 & 38.5 & 34.9
          & 75.5 & 44.9 & 79.4 & 39.4 & 65.4 & 58.0 \\
        UME-R1 (Qwen2-VL-2B)
          & 64.8 & 62.8 & 67.6 & 77.2 & 66.6 & 44.3 & 51.2 & 32.9 & 39.7 & 42.2
          & 72.4 & 46.2 & 79.2 & 37.2 & 63.9 & 60.1 \\
        TTEs (Qwen2-VL-2B)
          & 67.9 & 66.6 & 70.2 & 84.1 & 70.1 & 47.3 & 49.1 & 33.2 & 32.1 & 41.3
          & 77.5 & 53.2 & 83.2 & 41.1 & 68.8 & 63.1 \\
        Embed-RL (Qwen3-VL-2B)
          & 62.8 & 67.9 & 68.6 & \textbf{90.4} & 69.2
          & \textbf{57.0} & 55.9 & \textbf{45.1} & \textbf{49.4} & \textbf{52.1}
          & 79.9 & 52.0 & \textbf{84.6} & 65.7 & \underline{74.1} & \underline{66.8} \\
        RIME (Qwen2-VL-2B)
          & 67.9 & 64.4 & 69.8 & 82.1 & 69.1
          & 48.0 & 52.1 & 33.6 & 39.2 & 43.7
          & 76.4 & 51.4 & 81.7 & 63.9 & 71.4 & 64.1 \\
        PLUME (Qwen2-VL-2B)
          & 66.5 & 59.2 & 67.6 & 79.7 & 66.3 & 45.0 & 52.3 & 33.5 & \underline{46.7} & 44.1
          & 72.1 & 49.8 & 78.1 & 57.4 & 67.5 & 61.6 \\
        UMER-R (Qwen2-VL-2B)
          & 64.9 & 68.4 & 69.7 & 74.1 & 68.5
          & 51.4 & 51.6 & 35.3 & 33.3 & 44.0
          & 75.1 & 48.3 & \textbf{84.6} & \textbf{67.9} & 71.8 & 63.9 \\
        \rowcolor[gray]{0.92}
        \textbf{ReWAM} (Qwen2-VL-2B)
          & \underline{68.8} & \underline{70.3} & \textbf{71.2} & 82.7 & \underline{71.6}
          & 53.0 & \underline{61.3} & 36.2 & 35.4 & 47.7
          & 77.4 & 51.8 & \underline{84.2} & 65.2 & 72.8 & 66.5 \\
        \rowcolor[gray]{0.92}
        \textbf{ReWAM} (Qwen3-VL-2B)
          & \textbf{68.9} & \textbf{70.7} & \underline{70.5} & \underline{86.0} & \textbf{71.9}
          & \underline{53.6} & \textbf{63.3} & \underline{38.0} & 39.9 & \underline{49.7}
          & 78.6 & \textbf{55.4} & \underline{84.2} & \underline{66.7} & \textbf{74.2} & \textbf{67.4} \\
        \midrule
        \multicolumn{17}{c}{\textit{$\geq$4B model size}} \\
        \midrule
        GME (Qwen2-VL-7B)
          & 57.7 & 34.7 & 71.2 & 59.3 & 56.0 & 37.4 & 50.4 & 28.4 & 38.2 & 38.6
          & \textbf{89.4} & 55.6 & 85.0 & 44.4 & 75.2 & 57.8 \\
        LamRA-2 (Qwen2-VL-7B)
          & 59.2 & 26.5 & 70.0 & 62.7 & 54.1 & 39.3 & 42.6 & 24.3 & 34.6 & 35.2
          & 22.0 & 11.5 & 37.4 & 21.0 & 23.9 & 40.4 \\
        LamRA-2.5 (Qwen2.5-VL-7B)
          & 51.7 & 34.1 & 66.9 & 56.7 & 52.4 & 32.9 & 42.6 & 23.2 & 37.6 & 33.7
          & 56.3 & 33.3 & 58.2 & 40.1 & 50.2 & 47.4 \\
        VLM2Vec (Qwen2-VL-7B)
          & 62.7 & 56.9 & 69.4 & 82.2 & 65.5 & 39.1 & 30.0 & 29.0 & 40.6 & 34.0
          & 56.9 & 9.4 & 59.1 & 38.1 & 46.4 & 52.3 \\
        VLM2Vec-V2 (Qwen2-VL-7B)
          & 65.7 & 61.5 & 70.0 & 85.2 & 68.1 & 45.9 & 33.9 & 27.6 & 39.3 & 36.4
          & 78.8 & 52.6 & 82.7 & 42.1 & 69.3 & 61.2 \\
        CAFe (LLaVA-OV-7B)
          & 63.6 & 61.7 & 69.1 & 87.6 & 67.6 & 35.8 & 58.7 & 34.4 & 39.5 & 42.4
          & 70.7 & 49.6 & 79.5 & 38.1 & 63.9 & 60.6 \\
        UME-R1 (Qwen2-VL-7B)
          & \underline{67.1} & 69.2 & 71.9 & 84.9 & 71.3 & 48.6 & 60.7 & 38.2 & 39.3 & 47.5
          & 75.7 & 50.5 & 83.7 & 37.6 & 67.1 & 64.5 \\
        RIME (Qwen2-VL-7B)
          & \textbf{70.3} & 71.7 & \underline{73.2} & 86.3 & \underline{73.4}
          & 52.6 & 62.0 & 38.4 & \underline{41.6} & 49.4
          & \underline{80.9} & 55.6 & \textbf{85.8} & 66.9 & 75.6 & 68.6 \\
        Embed-RL (Qwen3-VL-4B)
          & 63.7 & 70.5 & 71.3 & \textbf{91.4} & 70.1
          & \textbf{57.6} & 58.4 & \textbf{45.1} & \textbf{49.5} & \textbf{53.0}
          & 80.2 & 53.4 & 84.9 & 67.1 & 74.7 & 68.1 \\
        \rowcolor[gray]{0.92}
        \textbf{ReWAM} (Qwen3-VL-4B)
          & \textbf{70.3} & \underline{71.8} & 72.3 & \underline{88.4} & \underline{73.4}
          & 52.2 & \textbf{63.2} & \underline{39.5} & 40.4 & \underline{49.8}
          & 80.1 & \textbf{58.2} & \underline{85.7} & \textbf{68.1} & \textbf{75.8} & \underline{68.7} \\
        \rowcolor[gray]{0.92}
        \textbf{ReWAM} (Qwen2-VL-7B)
          & \textbf{70.3} & \textbf{72.5} & \textbf{74.2} & 87.8 & \textbf{74.2}
          & \underline{53.6} & \underline{62.6} & 38.6 & 38.6 & 49.5
          & \underline{80.9} & \underline{56.1} & \underline{85.7} & \underline{67.3}
          & \underline{75.7} & \textbf{69.0} \\
        \bottomrule
    \end{tabular}}
    \global\benchmarktablewidth=\wd\benchmarktablereference\relax
    \resizebox{\linewidth}{!}{\usebox{\benchmarktablereference}}
\end{table}

\section{Experiments}
\label{sec:experiments}

\subsection{Experimental setup}

\textbf{Implementation details.}\enspace
We train CoT-conditioned embedders based on Qwen2-VL (2B and 7B) and
Qwen3-VL (2B and 4B). Following Embed-RL \citep{jiang2026embed}, we adopt
Qwen3-VL-8B-Instruct as the reasoner.
\textit{For SFT,} we use the same training data and prompt templates as
\mbox{RIME~\citep{wu2026beyond}}. The embedders are trained on 1.4M query--target pairs
for one epoch with a learning rate of $5\times10^{-5}$ and a batch size of 512.
\textit{For RASD,} we train on 20K query--target pairs with hard-negative mining,
as described in Appendix~\ref{app:data}. We use a fixed
Qwen3.5-122B-A10B API as the multimodal analyzer to construct
privileged guidance. The reasoner is trained for one epoch with a learning
rate of $1\times10^{-6}$ and a batch size of
32 query--target pairs.
We sample $G=8$ CoTs per input and set $\epsilon_R=\epsilon_P=0.2$,
$\lambda_R=0.5$, and KL weight $\beta=0.01$.
\textit{For RAI,} the RASD-trained reasoner greedily generates 730K offline CoT rollouts
from randomly sampled training inputs. We train the Retrieval Utility Head
and Draft model for 10 epochs with a batch size of 128 and learning rates
of $3\times10^{-4}$ and $6\times10^{-4}$, respectively.
Following \citet{cheng2026dspark}, we set $\lambda_{\mathrm{ntp}}=0.1$,
$\lambda_{\mathrm{align}}=0.9$, $\lambda_{\mathrm{acc}}=1$, and $\gamma=4$
in Eq.~(\ref{eq:draft-objective}).
We set the stopping threshold to $\tau_{\mathrm{stop}}=0.01$.

\textbf{Baselines.}\enspace
Our comparisons cover both standard multimodal embedders and reasoning-enhanced
methods. Standard embedding baselines include EVA-CLIP \citep{sun2023eva},
OpenCLIP \citep{cherti2023reproducible}, VISTA \citep{zhou2024vista},
E5-V \citep{jiang2024e5}, ColPali \citep{faysse2025colpali},
GME \citep{zhang2025bridging}, VLM2Vec and VLM2Vec-V2
\citep{jiang2024vlm2vec, meng2025vlm2vec}, LamRA \citep{liu2025lamra}, and
CAFe \citep{yu2025cafe}. Reasoning-enhanced baselines include UME-R1
\citep{lan2026ume}, RIME \citep{wu2026beyond},
TTEs \citep{cui2026think}, Embed-RL \citep{jiang2026embed}, and UMER-R
\citep{chen2026umer}. We also include PLUME \citep{he2026plume} as a
latent-reasoning baseline. On MRMR, we include MM-Embed \citep{lin2025mm},
Ops-MM-Embed \citep{opensearchai2025opsmm}, and LaME \citep{wu2026lame}.

\textbf{Evaluation.}\enspace
MMEB-V2 \citep{meng2025vlm2vec} covers 78 datasets for general-purpose
multimodal retrieval, including 36 image, 18 video, and 24 visual-document (visdoc)
datasets. MRMR \citep{zhang2026mrmr} assesses reasoning-intensive retrieval
across 11 subtasks in Knowledge, Theorem, and Contradiction.
Following \citet{wu2026beyond},
we report Hit@1 for image and video tasks and NDCG@5 for visdoc
tasks on MMEB-V2. For MRMR,
we report NDCG@10 for all subtasks except Negation, which uses Hit@1.
We measure inference speed on a single H800 GPU.

\subsection{Main results}

\textbf{Universal multimodal retrieval.}\enspace
ReWAM achieves the best overall scores of 67.4 and 69.0 in the two
model-size groups of Table~\ref{tab:mmeb-v2-results}, outperforming both
standard and reasoning-enhanced multimodal embedding methods.
With Qwen2-VL-2B, ReWAM surpasses VLM2Vec-V2 by 8.5 points overall.
Using the same SFT data and prompt templates as RIME, it gains 2.4 points
overall and 2.5, 4.0, and 1.4 points on image, video, and visdoc
tasks, respectively. It also exceeds the latent-reasoning method PLUME
by 4.9 points overall while retaining explicit CoT.
ReWAM further outperforms Embed-RL by 0.6 points at both the 2B and 4B
scales with matched embedding backbones.
These comparisons suggest that retrieval-aware supervision improves the retrieval utility of explicit CoT.

\ifrewamsubmissionlayout\suppressfloats[t]\fi
\begin{table}[t]
    \centering
    \caption{Results on the MRMR benchmark, covering Art, Medicine (Med.),
    Science (Sci.), Humanities (Hum.), Math, Physics (Phy.), Engineering
    (Eng.), Business (Bus.), Negation (Neg.), Design (Des.), and Traffic
    (Tra.). Best and second-best scores are bolded and underlined, respectively.}
    \label{tab:mrmr-results}
    \vspace{\baselineskip}
    \setlength{\tabcolsep}{3pt}
    \renewcommand{\arraystretch}{1}
    \newcommand{\mrmrtabular}{%
    \begin{tabular}{llccccccccccccc}
        \toprule
        \multirow{2}{*}{Model} & \multirow{2}{*}{Backbone}
        & \multirow{2}{*}{Size} & \multicolumn{4}{c}{Knowledge}
        & \multicolumn{4}{c}{Theorem} & \multicolumn{3}{c}{Contradiction}
        & \multirow{2}{*}{All} \\
        \cmidrule(lr){4-7}\cmidrule(lr){8-11}\cmidrule(lr){12-14}
        & & & Art & Med. & Sci. & Hum. & Math & Phy. & Eng. & Bus.
        & Neg. & Des. & Tra. & \\
        \midrule
        EVA-CLIP & EVA-ViT & 0.4B
          & 10.2 & 13.5 & 26.1 & 12.9
          & 6.2 & 10.5 & 9.3 & 11.7
          & 8.5 & 4.4 & 5.4 & 10.8 \\
        OpenCLIP & ViT-G/14 & 1B
          & 56.0 & 17.9 & 33.2 & 22.0
          & 5.7 & 5.0 & 7.0 & 9.7
          & 13.0 & 8.1 & 12.4 & 17.3 \\
        VISTA & Qwen2-VL & 2B
          & 21.3 & 27.8 & 32.6 & 17.0
          & 18.8 & 17.1 & 17.3 & 28.6
          & \underline{20.0} & 20.2 & 9.4 & 20.9 \\
        E5-V & LLaVA-Next & 8B
          & 25.1 & 11.7 & 16.6 & 10.8
          & 2.1 & 3.4 & 2.5 & 5.2
          & 11.5 & 3.7 & 2.1 & 8.6 \\
        VLM2Vec & Qwen2-VL & 7B
          & 53.5 & 22.4 & 36.7 & 24.0
          & 2.1 & 2.8 & 2.8 & 2.9
          & 11.5 & 5.6 & 18.3 & 18.1 \\
        ColPali & PaliGemma & 3B
          & 36.1 & 29.9 & 42.7 & 29.2
          & 5.7 & 14.8 & 12.0 & 24.6
          & \textbf{28.5} & 19.4 & 18.2 & 23.7 \\
        GME & Qwen2-VL & 7B
          & 54.3 & 40.1 & 46.8 & 45.6
          & 28.8 & 36.0 & 30.2 & 45.1
          & 15.0 & 26.3 & 29.6 & 36.2 \\
        MM-Embed & NV-Embed & 8B
          & 65.6 & 53.0 & 63.5 & 62.8
          & 23.6 & 30.8 & 27.4 & 44.9
          & 7.0 & 23.8 & 34.9 & 39.8 \\
        Ops-MM-Embed & Qwen2-VL & 7B
          & \underline{79.3} & 52.5 & 70.0 & 67.8
          & 27.7 & 39.5 & 30.1 & 52.3
          & 8.0 & 55.9 & \underline{45.8} & 48.1 \\
        Embed-RL & Qwen3-VL & 4B
          & 73.5 & 53.1 & 60.5 & 69.8
          & 32.9 & 45.6 & 35.1 & 52.6
          & 6.0 & 54.3 & 31.7 & 46.8 \\
        UME-R1 & Qwen2-VL & 7B
          & 77.8 & 55.7 & 72.9 & 64.1
          & 27.2 & 39.2 & 32.2 & 47.8
          & 7.5 & 61.9 & 41.7 & 48.0 \\
        LaME & Qwen2-VL & 7B
          & 73.4 & 58.2 & \textbf{73.8} & 65.6
          & 29.5 & 44.4 & 36.4 & 52.2
          & 8.5 & \textbf{64.9} & 40.9 & 49.8 \\
        RIME & Qwen2-VL & 7B
          & 76.8 & \underline{58.5} & \underline{73.6} & 71.4
          & 29.2 & 43.0 & 35.6 & 51.8
          & 8.5 & \underline{64.1} & 39.5 & 50.2 \\
        \midrule
        \rowcolor[gray]{0.92}
        \textbf{ReWAM} & Qwen3-VL & 4B
          & \textbf{79.9} & \textbf{60.0} & 73.0 & \textbf{74.8}
          & \textbf{40.2} & \textbf{52.2} & \textbf{43.1} & \textbf{56.0}
          & 6.5 & 46.7 & \textbf{53.5} & \textbf{53.3} \\
        \rowcolor[gray]{0.92}
        \textbf{ReWAM} & Qwen2-VL & 7B
          & 77.8 & 57.4 & 72.3 & \underline{73.3}
          & \underline{33.3} & \underline{45.8} & \underline{40.5} & \underline{53.6}
          & 8.0 & 60.1 & 43.3 & \underline{51.4} \\
        \bottomrule
    \end{tabular}}
    \sbox{\benchmarktablereference}{\mrmrtabular}
    \addtolength{\tabcolsep}{\dimexpr(\benchmarktablewidth-\wd\benchmarktablereference)/30\relax}
    \resizebox{\linewidth}{!}{\mrmrtabular}
\end{table}

\textbf{Reasoning-intensive retrieval.}\enspace
ReWAM achieves the top two overall scores on MRMR
(Table~\ref{tab:mrmr-results}), with 53.3 and 51.4 for the 4B and 7B
models, respectively. ReWAM-4B surpasses RIME, LaME, Embed-RL, and
UME-R1 by 3.1, 3.5, 6.5, and 5.3 points, respectively. It leads on eight
of the eleven subtasks. In Knowledge tasks,
it achieves the best results in Art, Medicine, and Humanities.
The gains are particularly pronounced in Theorem tasks, where it ranks first
on all four subtasks, exceeding the strongest prior results by 7.3,
6.6, 6.7, and 3.4 points in Math, Physics, Engineering, and Business,
respectively. In Contradiction tasks, it reaches 53.5 on Traffic, improving
over the best prior result by 7.7 points. These gains support the value of guiding
input-only reasoning toward input-grounded evidence that distinguishes relevant
targets in reasoning-intensive retrieval.

\textbf{Retrieval quality and inference efficiency.}\enspace
In Fig.~\ref{fig1}(e), ReWAM achieves $3.75$~\mbox{\textit{samples/s}}, versus
$0.75$ for decoupled Embed-RL and $3.41$ for latent-reasoning PLUME.
ReWAM thus delivers approximately $5\times$ and $1.1\times$ their throughput,
respectively, while retaining explicit CoT and outperforming PLUME on MMEB-V2
with the same embedding backbone.

\ifrewamsubmissionlayout\else\FloatBarrier\fi
\subsection{Ablation studies and analysis}

\ifrewamsubmissionlayout\Needspace{32\baselineskip}\fi

\begin{wraptable}[18]{\rewamablationplacement}{0.41\textwidth}
    \vspace{-\intextsep}
    \setlength{\belowcaptionskip}{\abovecaptionskip}
    \setlength{\abovecaptionskip}{0pt}
    \centering
    \setlength{\tabcolsep}{2.5pt}
    \renewcommand{\arraystretch}{0.92}
    \sbox{\compacttablereference}{%
        \begin{tabular}{>{\small}lcccc}
            Trajectory-level & Image & VisDoc & Video & ALL \\
            + Grounded evidence & Image & VisDoc & Video & ALL \\
            + Decision boundary & Image & VisDoc & Video & ALL \\
            + Claim verification & Image & VisDoc & Video & ALL \\
        \end{tabular}}
    \setlength{\compacttablewidth}{\wd\compacttablereference}
    \caption{Main ablation studies using the
Qwen2-VL-2B embedder.}
    \label{tab:main-ablation}
    \resizebox{\linewidth}{!}{%
    \begin{tabular*}{\compacttablewidth}{@{\extracolsep{\fill}\hspace{\tabcolsep}}>{\small}lcccc}
        \toprule
        \multicolumn{1}{l}{Model} & Image & VisDoc & Video & ALL \\
        \midrule
        ReWAM & \textbf{71.6} & \textbf{72.8} & \textbf{47.7} & \textbf{66.5} \\
        w/o RASD & 70.8 & 72.1 & 46.9 & 65.7 \\
        w/o CoT & 67.0 & 70.1 & 42.6 & 62.3 \\
        w/ raw input & 68.5 & 70.4 & 42.5 & 63.1 \\
        \bottomrule
    \end{tabular*}}

    \vspace{\dimexpr\floatsep+10pt\relax}

    \caption{RASD training ablation. Best results are bolded.}
    \label{tab:rasd-ablation}
    \resizebox{\linewidth}{!}{%
    \begin{tabular*}{\compacttablewidth}{@{\extracolsep{\fill}\hspace{\tabcolsep}}>{\small}lcccc}
        \toprule
        \multicolumn{1}{l}{Model} & Image & VisDoc & Video & ALL \\
        \midrule
        Trajectory-level & 71.1 & 72.4 & 47.3 & 66.0 \\
        $\lambda_R=0.1$ & 71.4 & 72.4 & 47.5 & 66.2 \\
        $\lambda_R=0.3$ & 71.4 & 72.6 & 47.6 & 66.3 \\
        $\lambda_R=0.5$ & \textbf{71.6} & \textbf{72.8} & \textbf{47.7} & \textbf{66.5} \\
        $\lambda_R=0.7$ & 71.4 & 72.7 & 47.7 & 66.3 \\
        $\lambda_R=1.0$ & 71.6 & 72.7 & 47.5 & 66.4 \\
        \bottomrule
    \end{tabular*}}

    \par
\end{wraptable}

\textbf{Ablation on Core Components.}\enspace
We compare ReWAM with three variants in Table~\ref{tab:main-ablation}.
The \textit{w/ raw input} variant trains an embedder on raw
multimodal content alone. The \textit{w/o CoT} variant retains the trained
CoT-conditioned embedder but omits CoT during encoding. The \textit{w/o RASD} variant omits Retrieval-Aware Self-Distillation.
ReWAM outperforms all three variants across all modalities, with overall
gains of 3.4, 4.2, and 0.8 points, respectively. Removing CoT from the trained
embedder yields a lower overall score than raw-input training, indicating reliance on
the reasoning context. RASD further improves retrieval with CoT, supporting
the value of directing reasoning toward retrieval-relevant distinctions.

\textbf{Effectiveness of RASD.}\enspace
Despite using the same training data and configuration, trajectory-level GRPO
achieves lower overall scores than RASD across all modulation strengths
(Table~\ref{tab:rasd-ablation}). Even a small modulation
strength of $\lambda_R=0.1$ raises the overall score from 66.0 to 66.2.
Evidence-guided weights differentiate token updates within a CoT and
favor input-supported, retrieval-relevant reasoning. Even weak modulation
therefore introduces guidance absent from a uniform trajectory-level signal.
Performance peaks at 66.5 with $\lambda_R=0.5$. Further increasing
$\lambda_R$ to 0.7 or 1.0 yields 66.3 and 66.4, respectively, both above the
trajectory-level baseline. These results demonstrate the robustness of RASD
to the choice of modulation strength.
We further examine the contribution of Contrastive Evidence Privileged
Information (CEPI) by progressively adding its components to trajectory-level
supervision, as shown in Table~\ref{tab:cepi-ablation}. Grounded evidence supplies input-supported
facts, while decision boundaries distinguish the positive from hard
negatives. Claim verification checks the sampled CoT against this guidance
to identify supported and conflicting statements, completing CEPI.
The full configuration raises the overall score from 66.0 to 66.5, with
gains of 0.5, 0.4, and 0.4 points on image, visdoc, and video tasks,
respectively. These gains highlight the benefit of ReWAM
over trajectory-level supervision.

\WFclear
\Needspace{12\baselineskip}
\begin{wraptable}{r}{0.41\textwidth}
    \vspace{-\intextsep}
    \setlength{\belowcaptionskip}{\abovecaptionskip}
    \setlength{\abovecaptionskip}{0pt}
    \centering
    \setlength{\tabcolsep}{2.5pt}
    \renewcommand{\arraystretch}{0.92}
    \sbox{\compacttablereference}{%
        \begin{tabular}{>{\small}lcccc}
            Trajectory-level & Image & VisDoc & Video & ALL \\
            + Grounded evidence & Image & VisDoc & Video & ALL \\
            + Decision boundary & Image & VisDoc & Video & ALL \\
            + Claim verification & Image & VisDoc & Video & ALL \\
        \end{tabular}}
    \setlength{\compacttablewidth}{\wd\compacttablereference}
    \caption{Ablation of Contrastive Evidence Privileged Information (CEPI).
    Components are added cumulatively.}
    \label{tab:cepi-ablation}
    \resizebox{\linewidth}{!}{%
    \begin{tabular*}{\compacttablewidth}{@{\extracolsep{\fill}\hspace{\tabcolsep}}>{\small}lcccc}
        \toprule
        \multicolumn{1}{l}{Variant} & Image & VisDoc & Video & ALL \\
        \midrule
        Trajectory-level & 71.1 & 72.4 & 47.3 & 66.0 \\
        + Grounded evidence & 71.4 & 72.7 & 47.5 & 66.3 \\
        + Decision boundary & 71.5 & 72.8 & 47.6 & 66.4 \\
        + Claim verification & \textbf{71.6} & \textbf{72.8} & \textbf{47.7} & \textbf{66.5} \\
        \bottomrule
    \end{tabular*}}
    \par
\end{wraptable}

\newcommand{\rewamraifloat}[1][!t]{%
\begin{table}[#1]
    \centering
    \caption{Ablation of RAI on MMEB-V2. We report Hit@1 for Image and Video,
    NDCG@5 for VisDoc, the average number of CoT tokens, and throughput
    (samples/s).}
    \label{tab:rai-efficiency-ablation}
    \vspace{\baselineskip}
    \resizebox{\linewidth}{!}{%
    \begin{tabular}{lcccccccccc}
        \toprule
        \multirow{2}{*}{Variants} & \multirow{2}{*}{\shortstack{Retrieval\\Utility Head}}
        & \multicolumn{3}{c}{Image}
        & \multicolumn{3}{c}{VisDoc}
        & \multicolumn{3}{c}{Video} \\
        \cmidrule(lr){3-5}\cmidrule(lr){6-8}\cmidrule(l){9-11}
        & & Hit@1 & Avg. tokens & Samples/s
          & NDCG@5 & Avg. tokens & Samples/s
          & Hit@1 & Avg. tokens & Samples/s \\
        \midrule
        Baseline & $\times$ & 71.6 & 136.3 & 1.2 & 72.8 & 131.4 & 1.1 & 47.7 & 143.5 & 1.1 \\
        \midrule
        \multirow{2}{*}{$K=4$}
          & $\times$     & 71.5 & 137.4 & 2.5 & 72.7 & 132.2 & 2.2 & 47.9 & 144.1 & 2.3 \\
          & $\checkmark$
          & \cellcolor[gray]{0.92}71.4
          & \cellcolor[gray]{0.92}87.9
          & \cellcolor[gray]{0.92}4.1
          & \cellcolor[gray]{0.92}72.5
          & \cellcolor[gray]{0.92}97.1
          & \cellcolor[gray]{0.92}3.1
          & \cellcolor[gray]{0.92}47.4
          & \cellcolor[gray]{0.92}102.0
          & \cellcolor[gray]{0.92}2.9 \\
        \midrule
        \multirow{2}{*}{$K=7$}
          & $\times$     & 71.4 & 137.3 & 2.9 & 72.8 & 132.1 & 2.3 & 47.7 & 144.5 & 2.5 \\
          & $\checkmark$
          & \cellcolor[gray]{0.92}71.3
          & \cellcolor[gray]{0.92}87.9
          & \cellcolor[gray]{0.92}4.5
          & \cellcolor[gray]{0.92}72.5
          & \cellcolor[gray]{0.92}97.3
          & \cellcolor[gray]{0.92}3.3
          & \cellcolor[gray]{0.92}47.3
          & \cellcolor[gray]{0.92}101.8
          & \cellcolor[gray]{0.92}3.4 \\
        \bottomrule
    \end{tabular}}

    \vspace{\floatsep}
    \begin{minipage}{\linewidth}
        \centering
        \includegraphics[width=\linewidth]{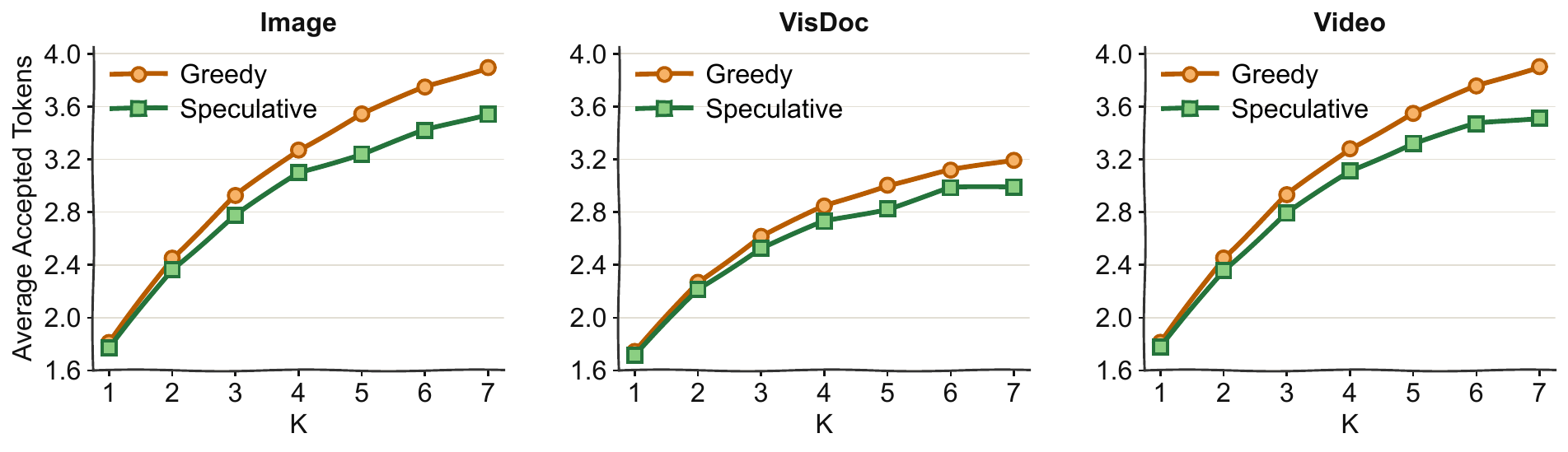}
        \captionof{figure}{Ablation of Retrieval-Adaptive Inference (RAI).
        Average accepted tokens as draft length $K$ varies under greedy
        and speculative acceptance on MMEB-V2. Counts include the bonus token.}
        \label{fig:rai-ablation}
    \end{minipage}
\end{table}
}
\ifrewamsubmissionlayout\rewamraifloat\fi

\textbf{Effectiveness of RAI.}\enspace
To evaluate RAI's effectiveness in accelerating retrieval-grounded reasoning,
we vary the number of draft tokens per block ($K=4$ or $7$) and ablate
the Retrieval Utility Head (Table~\ref{tab:rai-efficiency-ablation}).
Speculative decoding alone yields $2.0$--$2.4\times$ the baseline throughput
with nearly unchanged CoT lengths and retrieval scores.
Adding the head reduces CoT length by 26--36\% and further raises
throughput by 26--64\%.
Increasing $K$ from 4 to 7 raises full-RAI throughput from 4.1 to 4.5,
3.1 to 3.3, and 2.9 to 3.4 samples/s on image, visdoc, and video
tasks, respectively. At $K=7$, RAI averages a $3.3\times$ speedup across modalities,
with only a 0.3--0.4-point drop in retrieval scores. This speedup supports practical
CoT-based retrieval at scale.

\ifrewamsubmissionlayout\else
  \WFclear
  \rewamraifloat[!hb]
\fi

Fig.~\ref{fig:rai-ablation} shows continued gains in accepted-token counts up to
$K=7$, especially under greedy verification.  This suggests further speedup
potential from scaling draft-training data and block length.
We further compare greedy and speculative verification, with greedy verification
accepting more tokens across all three modalities. This gap may reflect stronger
draft--target top-1 agreement encouraged by greedy-generated training traces.
Such agreement suffices for greedy verification, whereas speculative verification
must preserve the target distribution and may therefore reject a matching token
when the draft overestimates its probability \citep{hu2025towards}.

\ifrewamsubmissionlayout\else\FloatBarrier\fi
\section{Conclusion}
\label{sec:conclusion}

We presented ReWAM for retrieval-grounded reasoning in universal multimodal
embeddings. Retrieval-Aware Self-Distillation (RASD) conditions an on-policy
self-teacher on input-supported contrastive evidence to refine token-level
credit assignment. Retrieval-Adaptive Inference (RAI) predicts remaining
retrieval utility to terminate unproductive reasoning and accelerates token
generation through speculative decoding. ReWAM achieves state-of-the-art
performance on MMEB-V2 and MRMR with up to $5\times$ the throughput of
competitive explicit-CoT methods. These advances make high-quality explicit
reasoning more practical for corpus-scale multimodal retrieval.

\section*{AI Use Statement}

AI tools assisted with drafting and revising the manuscript. The authors
carefully reviewed and revised every paragraph and take full responsibility
for the final content.

\section*{Ethics Statement}

We conduct experiments on existing multimodal datasets and do not collect
new data from human participants. As with other retrieval systems, ReWAM
may inherit biases from pretrained models and training data. Downstream
applications should account for these biases and protect the privacy of
indexed content.

\section*{Reproducibility Statement}

Section~\ref{sec:method} specifies the SFT, RASD, and RAI objectives and
procedures. Appendix~\ref{app:data} describes the training data and its
construction. Appendix~\ref{app:implementation} provides training settings,
the RAI algorithm, and evaluation protocols. Additional results appear in
Appendix~\ref{app:experimental-results}, and CEPI prompt templates in
Appendix~\ref{app:analyzer-prompts}. We will make the code publicly available
to support reproducibility.

\bibliography{iclr2027_conference}
\bibliographystyle{iclr2027_conference}

\clearpage
\appendix

\startcontents[appendix]
\pdfbookmark[0]{Appendix Contents}{app:contents}
\section*{Appendix Contents}
\begingroup
\hypersetup{hidelinks,linktoc=all}
\printcontents[appendix]{}{1}[2]{}
\endgroup
\clearpage

\raggedbottom
\section{Data Construction}
\label{app:data}
\subsection{SFT Data}
\label{app:sft-data}

We use 1.4M positive query--target pairs with the same SFT data and prompt
templates as \citet{wu2026beyond}, covering image, video, and visual-document
tasks. Each query and target is encoded with its original multimodal content
and corresponding offline CoT. These paired representations provide the
positive examples for bidirectional contrastive learning. The CoTs serve
as conditioning context for the embedder.

\subsection{RASD Data Construction}
\label{app:rasd-data}

Our RASD data consist of positive query--target pairs and hard-negative
relationships between them. We use the VLM2Vec-V2 training datasets
\citep{meng2025vlm2vec} to identify difficult retrieval comparisons,
filter them by reasoning benefit and semantic validity,
and retain the original pairs needed for training.
Table~\ref{tab:rasd-data-composition} summarizes the final composition.

\paragraph{(1) Hard-negative mining.}
To focus RASD on retrieval decisions that direct embeddings do not
resolve reliably, we mine confusable candidates from the original
multimodal inputs. We first remove exact duplicate pairs within each
dataset, then use a frozen Qwen2-VL-2B embedder without CoT conditioning
to retrieve the top 50 eligible candidates from the same dataset by
cosine similarity, excluding known positives, duplicate inputs, and
shared media. Mining is bidirectional except for
classification datasets, which use only query-to-target retrieval.
An input is selected if its hardest negative outranks all known
positives or its smallest positive--negative similarity gap falls in
the lowest 20\% for its dataset and retrieval direction.
We retain up to three negatives per input and direction and form a
200,000-pair pool.

\paragraph{(2) Retrieval-gain filtering.}
Retrieval difficulty alone does not show whether CoT improves
discrimination. To select comparisons with a measurable reasoning
benefit, we evaluate the same input, positive, and mined negative with
and without CoT.
The frozen Qwen3-VL-8B-Instruct reasoner greedily generates one CoT
for each query and target independently, using the task- and
modality-specific prompts of \citet{wu2026beyond}.
A separate frozen Qwen2-VL-2B embedder trained with CoT conditioning
encodes each item both with its own CoT and without CoT.
The retrieval margin is the input's cosine similarity to the positive
minus that to the negative. Let $m_{\mathrm{CoT}}$ and
$m_{\mathrm{noCoT}}$ denote the CoT-trained embedder's margins with and
without CoT, respectively. A comparison is retained only when both
conditions hold:
\begin{equation}
    m_{\mathrm{CoT}}>10^{-6},\qquad
    m_{\mathrm{CoT}}-m_{\mathrm{noCoT}}>10^{-6}.
    \label{eq:rasd-data-filter}
\end{equation}
The first condition requires the positive to rank above the negative,
while the second requires CoT to improve the margin over the same
embedder's no-CoT baseline.

\paragraph{(3) Semantic verification.}
Mined negatives may include valid alternative matches, which would
introduce incorrect supervision if treated as negatives. We therefore
use the fixed Qwen3.5-122B-A10B model as an LLM judge to verify positive
and negative labels. It receives the original multimodal content of
the input item, its positive, and up to three retained hard negatives,
with their retrieval roles explicitly labeled. The judge returns a
validity decision for the positive and flags alternative positives,
semantic duplicates, and uncertain judgments among the negative
candidates. It also provides a self-reported confidence score for the
positive decision and each candidate assessment. We retain a
comparison only when the positive is judged valid, the negative has
none of these flags, and both confidence scores are at least 0.8.

\paragraph{(4) Training-set assembly.}
To retain informative negative comparisons within training batches,
we select positive pairs together with their verified hard-negative
relationships. We collect both original pairs from each verified
input--negative relationship and exclude the five classification
datasets to avoid shared-label false negatives in symmetric in-batch
training. We then select 20,000 pairs from the
remaining 20 datasets under modality quotas, preserving verified
hard-negative relationships among the selected pairs. These
relationships guide batch construction so that difficult competitors
can be evaluated together. RASD trains on the selected pairs with
freshly sampled CoTs; the offline CoTs above are used for data selection.

\subsection{RAI Data Construction}
\label{app:rai-data}

For training the Retrieval Utility Head and Draft model in RAI,
we sample distinct query--target pairs from the SFT data.
The RASD-trained Qwen3-VL-8B-Instruct reasoner greedily generates
CoTs separately for query and positive inputs using the task- and
modality-specific prompt templates of \citet{wu2026beyond}.
Within each sampled pool, CoTs shared by identical reasoner inputs are counted
once, while the original query--target associations are retained for
retrieval-label construction.
This yields 730,003 successful CoT sequences. We append a closing
\texttt{</think>} tag to sequences that lack one.
Table~\ref{tab:draft-data-composition} summarizes the final dataset
composition.

\par\addvspace{\intextsep}
\noindent\begin{minipage}{\linewidth}
    \centering
    \captionof{table}{Training data composition. Initial and SFT report the
    initial and filtered positive-pair counts reported by
    \citet{wu2026beyond}, respectively; RASD reports our selected pairs.
    Modality indicates the query-to-target direction, where $T$, $I$,
    and $V$ denote text, image, and video; text includes task instructions.}
    \label{tab:rasd-data-composition}
    \vspace{\baselineskip}
    \small
    \setlength{\tabcolsep}{6pt}
    \begin{tabular*}{\textwidth}{@{\extracolsep{\fill}}lrrrl@{}}
        \toprule
        Dataset & Initial & SFT & RASD & Modality \\
        \midrule
        \multicolumn{5}{l}{\textit{Image-based (MMEB-train)}} \\
        A-OKVQA & 17,000 & 15,304 & 384 & $T{+}I \rightarrow T$ \\
        CIRR & 26,000 & 22,189 & 416 & $T{+}I \rightarrow T{+}I$ \\
        ChartQA & 28,000 & 25,479 & 1,472 & $T{+}I \rightarrow T$ \\
        DocVQA & 40,000 & 37,049 & 1,408 & $T{+}I \rightarrow T$ \\
        HatefulMemes & 8,000 & 7,352 & 0 & $T{+}I \rightarrow T$ \\
        ImageNet-1K & 50,000 & 43,607 & 0 & $T{+}I \rightarrow T$ \\
        InfographicsVQA & 24,000 & 21,059 & 1,280 & $T{+}I \rightarrow T$ \\
        MSCOCO & 50,000 & 23,417 & 224 & $T{+}I \rightarrow T{+}I$ \\
        MSCOCO-i2t & 50,000 & 44,630 & 512 & $T{+}I \rightarrow T$ \\
        MSCOCO-t2i & 50,000 & 43,229 & 384 & $T \rightarrow T{+}I$ \\
        N24News & 49,000 & 40,960 & 0 & $T{+}I \rightarrow T$ \\
        NIGHTS & 16,000 & 12,615 & 256 & $T{+}I \rightarrow T{+}I$ \\
        OK-VQA & 9,000 & 8,022 & 224 & $T{+}I \rightarrow T$ \\
        SUN397 & 20,000 & 18,879 & 0 & $T{+}I \rightarrow T$ \\
        VOC2007 & 8,000 & 6,456 & 0 & $T{+}I \rightarrow T$ \\
        Visual7W & 50,000 & 44,382 & 512 & $T{+}I \rightarrow T$ \\
        VisDial & 50,000 & 39,537 & 672 & $T \rightarrow T{+}I$ \\
        VisualNews-i2t & 50,000 & 41,177 & 640 & $T{+}I \rightarrow T$ \\
        VisualNews-t2i & 50,000 & 39,823 & 608 & $T \rightarrow T{+}I$ \\
        WebQA & 17,000 & 12,978 & 288 & $T \rightarrow T{+}I$ \\
        Subtotal & 662,000 & 548,144 & 9,280 & Image-centric \\
        \midrule
        \multicolumn{5}{l}{\textit{Video-based (LLaVA-Hound)}} \\
        Caption Retrieval & 300,000 & 232,044 & 3,296 & $T{+}V \rightarrow T$ \\
        Video QA & 300,000 & 207,829 & 1,472 & $T{+}V \rightarrow T$ \\
        Video Retrieval & 300,000 & 244,327 & 3,232 & $T \rightarrow T{+}V$ \\
        Subtotal & 900,000 & 684,200 & 8,000 & Video-centric \\
        \midrule
        \multicolumn{5}{l}{\textit{Document-based}} \\
        ViDoRe & 100,000 & 90,415 & 1,728 & $T{+}I \rightarrow T$ \\
        VisRAG & 100,000 & 81,086 & 992 & $T \rightarrow I$ \\
        Subtotal & 200,000 & 171,501 & 2,720 & Document-centric \\
        \midrule
        \textbf{Total} & \textbf{1,762,000} & \textbf{1,403,845} & \textbf{20,000} & Multimodal \\
        \bottomrule
    \end{tabular*}
\end{minipage}
\par\addvspace{\intextsep}

\par\addvspace{\intextsep}
\noindent\begin{minipage}{\linewidth}
    \centering
    \captionof{table}{RAI data composition for Retrieval Utility Head and
    Draft model training. Counts report single-input CoT records after input
    deduplication within each sampled pool.
    Query and Positive denote each record's source role; percentages use
    all 730,003 sequences.}
    \label{tab:draft-data-composition}
    \vspace{\baselineskip}
    \small
    \setlength{\tabcolsep}{4pt}
    \begin{tabular*}{\textwidth}{@{\extracolsep{\fill}}lrrrr@{}}
        \toprule
        Dataset & Query & Positive & Total & Ratio \\
        \midrule
        \multicolumn{5}{l}{\textit{Image / VQA / classification}} \\
        A-OKVQA & 6,208 & 2,997 & 9,205 & 1.26\% \\
        CIRR & 8,448 & 6,045 & 14,493 & 1.99\% \\
        ChartQA & 9,472 & 4,681 & 14,153 & 1.94\% \\
        DocVQA & 11,648 & 7,984 & 19,632 & 2.69\% \\
        HatefulMemes & 2,944 & 4 & 2,948 & 0.40\% \\
        ImageNet-1K & 11,136 & 1,852 & 12,988 & 1.78\% \\
        InfographicsVQA & 9,344 & 4,136 & 13,480 & 1.85\% \\
        MSCOCO & 5,184 & 4,287 & 9,471 & 1.30\% \\
        MSCOCO (i2t) & 11,648 & 11,482 & 23,130 & 3.17\% \\
        MSCOCO (t2i) & 10,752 & 10,432 & 21,184 & 2.90\% \\
        N24News & 7,104 & 47 & 7,151 & 0.98\% \\
        NIGHTS & 7,040 & 5,491 & 12,531 & 1.72\% \\
        OK-VQA & 3,456 & 1,921 & 5,377 & 0.74\% \\
        SUN397 & 9,088 & 751 & 9,839 & 1.35\% \\
        VOC2007 & 3,200 & 40 & 3,240 & 0.44\% \\
        VisDial & 8,704 & 8,704 & 17,408 & 2.38\% \\
        Visual7W & 10,429 & 6,588 & 17,017 & 2.33\% \\
        VisualNews (i2t) & 8,576 & 8,576 & 17,152 & 2.35\% \\
        VisualNews (t2i) & 7,168 & 7,168 & 14,336 & 1.96\% \\
        WebQA & 7,296 & 4,182 & 11,478 & 1.57\% \\
        Subtotal & 158,845 & 97,368 & 256,213 & 35.10\% \\
        \midrule
        \multicolumn{5}{l}{\textit{Video / LLaVA-Hound}} \\
        LLaVA-Hound caption retrieval & 71,296 & 71,296 & 142,592 & 19.53\% \\
        LLaVA-Hound QA & 68,608 & 68,237 & 136,845 & 18.75\% \\
        LLaVA-Hound video retrieval & 65,152 & 65,152 & 130,304 & 17.85\% \\
        Subtotal & 205,056 & 204,685 & 409,741 & 56.13\% \\
        \midrule
        \multicolumn{5}{l}{\textit{Document retrieval}} \\
        ViDoRe & 20,992 & 16,336 & 37,328 & 5.11\% \\
        VisRAG & 11,617 & 15,104 & 26,721 & 3.66\% \\
        Subtotal & 32,609 & 31,440 & 64,049 & 8.77\% \\
        \midrule
        \textbf{Total} & \textbf{396,510} & \textbf{333,493}
            & \textbf{730,003} & \textbf{100.00\%} \\
        \bottomrule
    \end{tabular*}
\end{minipage}
\par\addvspace{\intextsep}

\clearpage

\section{Implementation and Evaluation Details}
\label{app:implementation}

\subsection{Training Details}
\label{app:training}
\label{app:rasd-optimization}

\paragraph{SFT Training Details.}
The embedders are initialized from Qwen2-VL (2B and 7B) and
Qwen3-VL (2B and 4B) and trained on 1.4M query--target pairs for one
epoch with a learning rate of $5\times10^{-5}$. Following the sub-batch
scheme of VLM2Vec \citep{jiang2024vlm2vec}, each contrastive sub-batch contains 128 pairs
drawn from a single dataset. The contrastive loss in
Eq.~(\ref{eq:contrastive-prelim}) is computed within each sub-batch with
temperature $\tau_E=0.02$,
while the global training batch contains 512 pairs. We use a cosine
learning-rate schedule with a warm-up ratio of 0.03. The maximum
sequence length is set to 8,192 tokens. For images and visual-document
pages, we set the minimum and maximum pixel counts to $16\times28\times28$
and $576\times28\times28$, respectively. For videos, the corresponding
per-frame limits are $256\times28\times28$ and $1024\times28\times28$.
We sample 4--8 frames per video. Images and frames outside these ranges
are resized to satisfy the pixel constraints. RASD and RAI use the same
visual preprocessing settings.

\paragraph{RASD Training Details.}
The SFT embedder is frozen, and the Qwen3-VL-8B-Instruct
reasoner is optimized on 20K query--target pairs for one epoch with a
learning rate of $1\times10^{-6}$. Each training batch contains 32
pairs from the same dataset, and each query or positive input produces
$G=8$ CoT rollouts with a maximum response length of 1,024 tokens.
Rollouts use temperature 1.0, top-$p$ 0.9, and top-$k$ $-1$.
The fixed Qwen3.5-122B-A10B analyzer uses the multimodal context of the
input, its positive, and retrieved hard negatives to extract input-grounded
evidence and discriminative boundaries. For each input, we retrieve the
\textit{top-3} in-batch hard negatives using the frozen embedder without CoT
conditioning. The analyzer then checks each sampled CoT against the
original multimodal context and extracted evidence, identifying
supported claims and incorrect or borrowed details. The resulting evidence,
decision boundaries, and claim-verification summaries constitute CEPI.
The analyzer uses the same visual preprocessing as SFT and generates
at most 512 tokens per response. We set the KL-penalty coefficient to
0.01, the format reward weight to $\lambda_f=0.01$, the margin reward
tolerance to $\epsilon_v=0.01$, both clipping thresholds to
$\epsilon_R=\epsilon_P=0.2$, and the token-modulation strength to
$\lambda_R=0.5$. The CEPI prompt templates are provided in
Appendix~\ref{app:analyzer-prompts}.

\paragraph{RAI Training Details.}
The RASD-trained reasoner remains frozen while the Retrieval Utility
Head and Draft model are trained on the CoT collection in
Appendix~\ref{app:rai-data}. Both modules are trained for 10 epochs with
a batch size of 128 CoT sequences and learning rates of
$3\times10^{-4}$ and $6\times10^{-4}$, respectively.
The Retrieval Utility Head is an MLP with hidden widths 64 and 128
and a sigmoid output, trained with the SmoothL1 objective in
Eq.~(\ref{eq:halting-loss}); its stopping threshold is $\tau_{\mathrm{stop}}=0.01$.
The Draft model architecture follows \citet{chen2026dflash} and
\citet{cheng2026dspark}.
We train it with the objective in Eq.~(\ref{eq:draft-objective}),
setting $\lambda_{\mathrm{ntp}}=0.1$, $\lambda_{\mathrm{align}}=0.9$,
$\lambda_{\mathrm{acc}}=1$, and $\gamma=4$.

\subsection{Analysis of Retrieval-Aware Self-Distillation}
\label{app:rasd-analysis}

Trajectory-level retrieval feedback does not identify which reasoning claims
are supported by the input and distinguish the positive from hard negatives.
RASD uses CEPI to provide token-level guidance for input-only reasoning;
we establish bounds on its modulation and its effect on the policy gradient.

Consider one policy update with fixed sampled CoTs $o^{(i)}$, privileged
guidance $c^{(i)}$, finite advantages $A_i$, and old policy $\pi_{\mathrm{old}}$.
The two log-probabilities defining $\Delta_t^i$ use the same pre-update reasoner
parameters and token prefix. As in Section~\ref{sec:rasd}, $A_i$,
$\Delta_t^i$, $s_t^i$, and $w_t^i$ are held constant when differentiating
the loss; the policy surrogate differentiates only the input-only likelihood
in $\rho_t^i$.

\begin{proposition}[Bounded modulation and monotonicity]
\label{prop:rasd-modulation}
Suppose $0<\epsilon_R<1$, $\lambda_R\geq0$, and
$\lambda_R\epsilon_R<1$. The weights and modulated advantages in
Eq.~(\ref{eq:rasd-weight}) satisfy
\begin{equation}
\begin{gathered}
    1-\lambda_R\epsilon_R\leq w_t^i\leq1+\lambda_R\epsilon_R,
    \qquad
    |\hat A_t^i-A_i|\leq\lambda_R\epsilon_R|A_i|,\\
    \operatorname{sign}(\hat A_t^i)=\operatorname{sign}(A_i).
\end{gathered}
\label{eq:rasd-modulation-bounds}
\end{equation}
For fixed $A_i$, let $F_i(d)$ be the modulated advantage obtained by
setting $s_t^i=\exp(\operatorname{sign}(A_i)d)$ in Eq.~(\ref{eq:rasd-weight}).
Then $F_i$ is nondecreasing and $F_i(0)=A_i$.
\end{proposition}

\begin{proof}
The clipped term lies in $[-\epsilon_R,\epsilon_R]$, giving
$|w_t^i-1|\leq\lambda_R\epsilon_R$ and hence
$|\hat A_t^i-A_i|=|A_i||w_t^i-1|
\leq\lambda_R\epsilon_R|A_i|$.
Since $w_t^i\geq1-\lambda_R\epsilon_R>0$, multiplying by $w_t^i$
preserves the sign of $A_i$.
For $A_i\neq0$, within the unclipped interval,
\begin{equation}
    F_i'(d)=\lambda_R|A_i|
    \exp\!\left(\operatorname{sign}(A_i)d\right)\geq0.
    \label{eq:rasd-modulation-monotonicity}
\end{equation}
Outside this interval, $F_i$ is constant in each clipped region.
Continuity at the clipping boundaries therefore establishes monotonicity.
For $A_i=0$, $F_i$ is identically zero.
At $d=0$, $s_t^i=1$ and $w_t^i=1$, so $F_i(0)=A_i$.
\end{proof}

For $\lambda_R>0$ and $A_i\neq0$, positive teacher feedback ($\Delta_t^i>0$) increases the advantage
coefficient: it strengthens reinforcement for $A_i>0$ and weakens penalties
for $A_i<0$, with the reverse effect for negative feedback.
With $\epsilon_R=0.2$ and $\lambda_R=0.5$, weights lie in $[0.9,1.1]$,
so the advantage changes by at most $10\%$ of $|A_i|$.
If $A_i=0$, the modulated advantage is also zero regardless of $\Delta_t^i$.

To analyze the policy update, define the per-token clipped surrogate
\begin{equation}
    S(\rho,A)=\min\!\left(
    \rho A,\operatorname{clip}(\rho,1-\epsilon_P,1+\epsilon_P)A\right).
    \label{eq:rasd-base-surrogate}
\end{equation}
Let $\ell_{i,t}^{0}=-S(\rho_t^i,A_i)$ and
$\ell_{i,t}^{\mathrm{RASD}}=-S(\rho_t^i,A_iw_t^i)$ use identical
samples and rewards. Their averages
$\mathcal L_{\mathrm{surr}}^{0}=\mathbb E_{i,t}[\ell_{i,t}^{0}]$ and
$\mathcal L_{\mathrm{surr}}^{\mathrm{RASD}}
=\mathbb E_{i,t}[\ell_{i,t}^{\mathrm{RASD}}]$
use the fixed empirical averaging in Eq.~(\ref{eq:rasd-objective})
and exclude its KL regularizer.

\begin{proposition}[Gradient reweighting and perturbation bound]
\label{prop:rasd-surrogate}
Under the conditions of Proposition~\ref{prop:rasd-modulation}, with
$A_i$ and $w_t^i$ detached and $0<\epsilon_P<1$,
\begin{equation}
    \ell_{i,t}^{\mathrm{RASD}}=w_t^i\ell_{i,t}^{0},
    \qquad
    \nabla_\theta\ell_{i,t}^{\mathrm{RASD}}
    =w_t^i\nabla_\theta\ell_{i,t}^{0},
    \label{eq:rasd-gradient-reweighting}
\end{equation}
where the gradient identity holds wherever the surrogate is differentiable.
If all token terms are differentiable, then for any norm,
\begin{equation}
    \left\|\nabla_\theta\mathcal L_{\mathrm{surr}}^{\mathrm{RASD}}
    -\nabla_\theta\mathcal L_{\mathrm{surr}}^{0}\right\|
    \leq\lambda_R\epsilon_R\,
    \mathbb E_{i,t}\!\left[\left\|
    \nabla_\theta\ell_{i,t}^{0}\right\|\right].
    \label{eq:rasd-gradient-bound}
\end{equation}
\end{proposition}

\begin{proof}
For $w>0$, factoring $w$ out of both arguments of the minimum gives
$S(\rho,wA)=wS(\rho,A)$ and preserves the active clipping branch.
Since $w_t^i$ is detached, differentiating gives the gradient identity.
The triangle inequality and $|w_t^i-1|\leq\lambda_R\epsilon_R$ imply
\begin{align*}
    \left\|\mathbb E_{i,t}\!\left[
    (w_t^i-1)\nabla_\theta\ell_{i,t}^{0}\right]\right\|
    &\leq\mathbb E_{i,t}\!\left[
    |w_t^i-1|\left\|\nabla_\theta\ell_{i,t}^{0}\right\|\right]\\
    &\leq\lambda_R\epsilon_R\,
    \mathbb E_{i,t}\!\left[
    \left\|\nabla_\theta\ell_{i,t}^{0}\right\|\right].\qedhere
\end{align*}
\end{proof}

RASD therefore rescales each token's surrogate gradient without changing
its clipping branch; gradients suppressed by clipping remain zero.
The perturbation bound is scaled by the mean of individual token-gradient
norms. Adding identical KL terms preserves this bound.

\subsection{Retrieval-Adaptive Inference Algorithm}
\label{app:rai-implementation}

Algorithm~\ref{alg:rai} expands the offline training and online decoding
procedure in Section~\ref{sec:rai}. We use its query-side notation
$q,t^+,\mathcal T^-$; target-side processing is symmetric.

\paragraph{Offline rollouts and prefix labels.}
The frozen reasoner generates one complete greedy CoT $o$ of length $L$
per input. For query $q$, candidate embeddings are cached from the targets'
full CoTs and held fixed. Let $\mathcal N\subseteq\{1,\ldots,L\}$ denote
evaluated prefix lengths. Endpoints start at token 32 with stride 28,
shifting to a sentence boundary within the next 12 tokens when available;
at most seven are retained, including $L$.
For each $n\in\mathcal N$, Eq.~(\ref{eq:remaining-utility}) gives $m_n$,
$g_n$, and $u_n$ from $\mathbf e_q^n=E_\phi(q,o_{\leq n})$.
The baseline $m_0$ is computed from the query embedding $E_\phi(q)$ without
CoT and the cached candidate embeddings.
Maxima range only over $\mathcal N$; all-zero gains give zero utility.
The implementation adds $10^{-6}$ to the denominator for numerical stability.

\paragraph{Training the two modules.}
Pairs $(h_n,u_n)$ train $H_\omega$ through Eq.~(\ref{eq:halting-loss}), using
the frozen reasoner's post-token states $h_n$.
The Draft model $D$ learns from offline CoT blocks via Eq.~(\ref{eq:draft-objective}),
with reference predecessor $t^\star_{n+k-1}$ and label $t^\star_{n+k}$.
Its auxiliary prediction $\hat a_{n+k}$ uses Draft states $z_{n+k}$ and the
preceding token to estimate acceptance; $\hat u_n$ uses target states $h_n$
to control CoT stopping.

\paragraph{Online draft verification and stopping.}
Starting from a target-generated anchor, $D$ proposes up to $K$ tokens
using Eq.~(\ref{eq:draft-generation}). The frozen reasoner verifies the block,
retains the longest prefix matching its greedy predictions, and supplies a
correction or bonus token. Verification returns this continuation and
valid target states, discarding states of rejected tokens and respecting
natural termination and the token budget. Corrected or bonus tokens are checked
once their own target states become available. For each valid state in token
order with $n\geq32$, $H_\omega(h_n)<\tau_{\mathrm{stop}}$ closes the CoT at
$n$ with \texttt{</think>}; otherwise decoding continues.
Online stopping reuses target states without candidate comparisons or embedder calls.

\begin{algorithm}[H]
\caption{RAI: training and inference}
\label{alg:rai}
\small
\algrenewcommand{\algorithmiccomment}[1]{\hfill\#\ #1}
\begin{algorithmic}[1]
\Require Frozen reasoner $\pi_\theta$ and embedder $E_\phi$; head $H_\omega$; Draft model $D$
\Require Offline groups $\mathcal G=\{(q,o,t^+,\mathcal T^-)\}$ with cached full-CoT target embeddings; query $q_\star$; draft length $K$; stopping threshold $\tau_{\mathrm{stop}}$
\Statex \textbf{Training}
\State $\mathcal D_H\gets\emptyset$ \Comment{Utility-head training set}
\For{$(q,o,t^+,\mathcal T^-)\in\mathcal G$}
    \State Select prefix endpoints $\mathcal N$ from $o$, including $L$. \Comment{Prefix sampling}
    \State Compute $m_0$ from $E_\phi(q)$ and cached targets. \Comment{No-CoT margin}
    \State Compute $m_n$ and $g_n$ using Eq.~(\ref{eq:remaining-utility}) for $n\in\mathcal N$. \Comment{Prefix retrieval gains}
    \State Extract $h_n$ at each selected endpoint. \Comment{Frozen reasoner's states}
    \For{$n\in\mathcal N$}
        \State $u_n\gets\dfrac{\max\{g_k\mid k\in\mathcal N,\ k\geq n\}-g_n}{\max\{g_k\mid k\in\mathcal N\}+10^{-6}}$ \Comment{Remaining utility}
        \State $\mathcal D_H\gets\mathcal D_H\cup\{(h_n,u_n)\}$ \Comment{State--utility pair}
    \EndFor
\EndFor
\State Train $H_\omega$ on $\mathcal D_H$ using Eq.~(\ref{eq:halting-loss}). \Comment{Utility regression}
\State Train $D$ on CoT blocks using Eq.~(\ref{eq:draft-objective}). \Comment{Draft model training}
\Statex \textbf{Inference}
\State Initialize $o$ and $\{h_n\}$ with $\pi_\theta(q_\star)$. \Comment{First greedy token}
\While{true}
    \State Predict $\hat u_n=H_\omega(h_n)$ for new verified states with $n\geq32$. \Comment{Stopping signal}
    \If{any $\hat u_n<\tau_{\mathrm{stop}}$}
        \State \Return $o_{\leq n}$ at the first such $n$, closed with \texttt{</think>}. \Comment{Truncate}
    \ElsIf{CoT is closed or the token budget is exhausted}
        \State \Return $o$. \Comment{Generation complete}
    \Else
        \State $D$ proposes up to $K$ tokens using Eq.~(\ref{eq:draft-generation}). \Comment{Draft proposals}
        \State Verify with $\pi_\theta$ and update $o$ and $\{h_n\}$. \Comment{Verified continuation}
    \EndIf
\EndWhile
\end{algorithmic}
\end{algorithm}

\subsection{Evaluation Protocol}
\label{app:evaluation}

MMEB-V2 \citep{meng2025vlm2vec} contains 78 datasets: 36 image,
18 video, and 24 visual-document datasets. We report Hit@1 for image
and video tasks and NDCG@5 for visual-document tasks.
MRMR \citep{zhang2026mrmr} contains 11 subtasks spanning Knowledge,
Theorem, and Contradiction. We report Hit@1 on Negation and NDCG@10
on the remaining subtasks. Overall scores are averaged over individual
datasets or subtasks. The main retrieval results and training ablations are
evaluated without RAI. RAI is evaluated separately in the inference-efficiency
experiments. In the ablation studies, throughput is averaged over three runs.
All inference speed measurements use a single H800 GPU with a batch size of 1.
Throughput covers the complete encoding pipeline, including CoT generation
and embedding extraction. To accelerate evaluation, we use vLLM to schedule reasoner inference.
We adopt greedy decoding with temperature 0.0, top-$p$ 1.0, and top-$k$ $-1$.
The reasoner's context limit is 8,192 tokens, with at most 7,168 prompt
tokens and 1,024 generated tokens. The embedder input length is capped
at 8,192 tokens. Our repetition checks use pattern lengths of 4--64
tokens and a minimum repeat count of 4.

\clearpage
\flushbottom
\section{Experimental Results}
\label{app:experimental-results}

%

\subsection{Additional Analysis}
\label{app:additional-analysis}

\paragraph{Fixed-length Truncation versus RAI.}
RAI accelerates CoT generation through speculative decoding and uses
predicted remaining retrieval utility to stop unproductive continuations.
To evaluate adaptive truncation, we compare RAI with fixed-budget and
full-CoT generation at $K=4$ and $K=7$, alongside an autoregressive
baseline. All variants use the same
Qwen2-VL-2B embedder and RASD-trained reasoner, with speculative-decoding
settings held fixed within each $K$.

As shown in Table~\ref{tab:fixed-truncation}, short fixed budgets yield
high throughput but sharply reduce retrieval quality. Increasing the
budget gradually restores performance at the cost
of speed, yet even the longest fixed budget leaves a 0.9-point
overall-score gap to the baseline.
RAI incurs the smallest overall-score loss among truncation strategies
at both draft lengths. Mean throughput reaches 3.4 and 3.7 samples/s
at $K=4$ and $K=7$, respectively, with only a 0.3--0.4-point score
decrease. At $K=7$, RAI is $1.6\times$ as fast as the longest fixed
budget while scoring 0.5 points higher. It thus provides the highest
throughput among configurations that preserve near-baseline retrieval
quality, demonstrating that retrieval-aware stopping offers a better
quality--efficiency trade-off than a uniform token budget.

\begin{table}[htbp]
    \centering
    \newcommand{\fixedgain}[1]{\,{\scriptsize\textcolor[rgb]{0.00,0.45,0.20}{$\uparrow$\,#1}}}
    \newcommand{\fixedloss}[1]{\,{\scriptsize\textcolor[rgb]{0.75,0.05,0.05}{$\downarrow$\,#1}}}
    \newcommand{\fixedequal}{\fixedgain{0.0}}
    \caption{Fixed-length truncation versus RAI. $L$ is the fixed CoT
    token budget and $K$ is the draft block length.
    Baseline uses autoregressive decoding without truncation.
    \textcolor[rgb]{0.00,0.45,0.20}{$\uparrow$} and
    \textcolor[rgb]{0.75,0.05,0.05}{$\downarrow$} indicate absolute increases
    and decreases from Baseline;
    \textcolor[rgb]{0.00,0.45,0.20}{$\uparrow 0.0$} denotes no change.}
    \label{tab:fixed-truncation}
    \vspace{\baselineskip}
    \small
    \resizebox{\textwidth}{!}{%
    \begin{tabular}{lccccccccc}
        \toprule
        \multirow{2}{*}{Stopping rule} & \multirow{2}{*}{$K$}
        & \multicolumn{2}{c}{Image} & \multicolumn{2}{c}{VisDoc}
        & \multicolumn{2}{c}{Video} & \multicolumn{2}{c}{ALL} \\
        \cmidrule(lr){3-4}\cmidrule(lr){5-6}\cmidrule(lr){7-8}\cmidrule(l){9-10}
        & & Hit@1 & Samples/s & NDCG@5 & Samples/s & Hit@1 & Samples/s
        & Score & Samples/s \\
        \midrule
        Baseline & -- & 71.6 & 1.2 & 72.8 & 1.1 & 47.7 & 1.1 & 66.5 & 1.1 \\
        \midrule
        Full CoT & 4 & 71.5\fixedloss{0.1} & 2.5\fixedgain{1.3} & 72.7\fixedloss{0.1} & 2.2\fixedgain{1.1} & 47.9\fixedgain{0.2} & 2.3\fixedgain{1.2} & 66.4\fixedloss{0.1} & 2.3\fixedgain{1.2} \\
        $L=16$ & 4 & 67.2\fixedloss{4.4} & 7.8\fixedgain{6.6} & 70.3\fixedloss{2.5} & 7.0\fixedgain{5.9} & 42.4\fixedloss{5.3} & 6.6\fixedgain{5.5} & 62.5\fixedloss{4.0} & 7.1\fixedgain{6.0} \\
        $L=32$ & 4 & 67.9\fixedloss{3.7} & 7.1\fixedgain{5.9} & 71.0\fixedloss{1.8} & 5.6\fixedgain{4.5} & 43.6\fixedloss{4.1} & 5.1\fixedgain{4.0} & 63.2\fixedloss{3.3} & 5.9\fixedgain{4.8} \\
        $L=64$ & 4 & 68.6\fixedloss{3.0} & 4.1\fixedgain{2.9} & 72.0\fixedloss{0.8} & 3.5\fixedgain{2.4} & 45.0\fixedloss{2.7} & 3.5\fixedgain{2.4} & 64.2\fixedloss{2.3} & 3.7\fixedgain{2.6} \\
        $L=128$ & 4 & 70.4\fixedloss{1.2} & 2.5\fixedgain{1.3} & 72.8\fixedequal & 2.0\fixedgain{0.9} & 46.5\fixedloss{1.2} & 1.7\fixedgain{0.6} & 65.6\fixedloss{0.9} & 2.1\fixedgain{1.0} \\
        \rowcolor[gray]{0.92}
        \textbf{RAI} & 4 & 71.4\fixedloss{0.2} & 4.1\fixedgain{2.9} & 72.5\fixedloss{0.3} & 3.1\fixedgain{2.0} & 47.4\fixedloss{0.3} & 2.9\fixedgain{1.8} & 66.2\fixedloss{0.3} & 3.4\fixedgain{2.3} \\
        \midrule
        Full CoT & 7 & 71.4\fixedloss{0.2} & 2.9\fixedgain{1.7} & 72.8\fixedequal & 2.3\fixedgain{1.2} & 47.7\fixedequal & 2.5\fixedgain{1.4} & 66.4\fixedloss{0.1} & 2.6\fixedgain{1.5} \\
        $L=16$ & 7 & 67.2\fixedloss{4.4} & 10.1\fixedgain{8.9} & 70.4\fixedloss{2.4} & 7.3\fixedgain{6.2} & 42.5\fixedloss{5.2} & 7.3\fixedgain{6.2} & 62.5\fixedloss{4.0} & 8.3\fixedgain{7.2} \\
        $L=32$ & 7 & 67.8\fixedloss{3.8} & 8.9\fixedgain{7.7} & 71.0\fixedloss{1.8} & 5.9\fixedgain{4.8} & 43.6\fixedloss{4.1} & 5.8\fixedgain{4.7} & 63.2\fixedloss{3.3} & 6.9\fixedgain{5.8} \\
        $L=64$ & 7 & 68.6\fixedloss{3.0} & 4.4\fixedgain{3.2} & 72.2\fixedloss{0.6} & 2.3\fixedgain{1.2} & 45.0\fixedloss{2.7} & 2.8\fixedgain{1.7} & 64.3\fixedloss{2.2} & 3.2\fixedgain{2.1} \\
        $L=128$ & 7 & 70.5\fixedloss{1.1} & 2.3\fixedgain{1.1} & 72.8\fixedequal & 2.2\fixedgain{1.1} & 46.3\fixedloss{1.4} & 2.5\fixedgain{1.4} & 65.6\fixedloss{0.9} & 2.3\fixedgain{1.2} \\
        \rowcolor[gray]{0.92}
        \textbf{RAI} & 7 & 71.3\fixedloss{0.3} & 4.5\fixedgain{3.3} & 72.5\fixedloss{0.3} & 3.3\fixedgain{2.2} & 47.3\fixedloss{0.4} & 3.4\fixedgain{2.3} & 66.1\fixedloss{0.4} & 3.7\fixedgain{2.6} \\
        \bottomrule
    \end{tabular}}
\end{table}
\FloatBarrier

\subsection{Detailed Results on MMEB-V2}
\label{app:mmeb-detailed}

Tables~\ref{tab:appendix-mmeb-image}
and~\ref{tab:appendix-mmeb-video-visdoc} report detailed results on all
78 MMEB-V2 datasets. We compare four ReWAM configurations with
UME-R1~\citep{lan2026ume}, Embed-RL~\citep{jiang2026embed},
RIME~\citep{wu2026beyond}, TTE$_s$~\citep{cui2026think},
and PLUME~\citep{he2026plume}, using reported baseline scores.
Embed-RL uses Qwen3-VL embedders; the other baselines use Qwen2-VL.
For ReWAM, Q2 and Q3 denote Qwen2-VL and Qwen3-VL, respectively.

\clearpage

\begin{table}[!htbp]
    \centering
    \caption{Detailed MMEB-V2 results: overall and task-group averages,
    followed by the 36 image datasets. Video and visual-document results
    continue in Table~\ref{tab:appendix-mmeb-video-visdoc}.
    Best and second-best scores are bolded and underlined, respectively.}
    \label{tab:appendix-mmeb-image}
    \vspace{\baselineskip}
    \scriptsize
    \setlength{\tabcolsep}{1pt}
    \begin{tabular}{@{}p{0.35\textwidth}*{7}{>{\centering\arraybackslash}p{\dimexpr0.053333\textwidth-2pt\relax}}>{\centering\arraybackslash}p{\dimexpr0.063333\textwidth-2pt\relax}*{4}{>{\centering\arraybackslash}p{\dimexpr0.053333\textwidth-2pt\relax}}@{}}
        \toprule
        \multirow{2}{*}{Dataset / task group}
            & \multicolumn{2}{c}{UME-R1}
            & \multicolumn{2}{c}{Embed-RL}
            & \multicolumn{2}{c}{RIME}
            & \multicolumn{1}{@{}c@{}}{TTE$_s$}
            & \multicolumn{1}{@{}c@{}}{PLUME}
            & \multicolumn{4}{c}{ReWAM} \\
        \cmidrule(lr){2-3}\cmidrule(lr){4-5}
        \cmidrule(lr){6-7}\cmidrule(lr){8-8}\cmidrule(lr){9-9}
        \cmidrule(l){10-13}
        & 2B & 7B & 2B & 4B & 2B & 7B & 2B & 2B
        & Q2-2B & Q3-2B & Q3-4B & Q2-7B \\
        \midrule
        \rowcolor[rgb]{0.91,0.95,0.99}
        Avg - All (78 tasks) & 60.1 & 64.5 & 66.8 & 68.1 & 64.1 & 68.6 & 63.1 & 61.6 & 66.5 & 67.4 & \underline{68.7} & \textbf{69.0} \\
        \midrule
        \rowcolor[rgb]{0.91,0.95,0.99}
        Avg - Image (36 tasks, Hit@1) & 66.6 & 71.3 & 69.2 & 70.1 & 69.1 & \underline{73.4} & 70.1 & 66.3 & 71.6 & 71.9 & \underline{73.4} & \textbf{74.2} \\
        \rowcolor[rgb]{0.91,0.95,0.99}
        Avg - Video (18 tasks, Hit@1) & 42.2 & 47.5 & \underline{52.1} & \textbf{53.0} & 43.7 & 49.4 & 41.3 & 44.1 & 47.7 & 49.7 & 49.8 & 49.5 \\
        \rowcolor[rgb]{0.91,0.95,0.99}
        Avg - VisDoc (24 tasks, NDCG@5) & 63.9 & 67.1 & 74.1 & 74.7 & 71.4 & 75.6 & 68.8 & 67.5 & 72.8 & 74.2 & \textbf{75.8} & \underline{75.7} \\
        \midrule
        \rowcolor[rgb]{0.99,0.95,0.90}
        I-CLS (10) & 64.8 & 67.1 & 62.8 & 63.7 & 67.9 & \textbf{70.3} & 67.9 & 66.5 & 68.8 & \underline{68.9} & \textbf{70.3} & \textbf{70.3} \\
        \rowcolor[rgb]{0.99,0.95,0.90}
        I-QA (10) & 62.8 & 69.2 & 67.9 & 70.5 & 64.4 & 71.7 & 66.6 & 59.2 & 70.3 & 70.7 & \underline{71.8} & \textbf{72.5} \\
        \rowcolor[rgb]{0.99,0.95,0.90}
        I-RET (12) & 67.6 & 71.9 & 68.6 & 71.3 & 69.8 & \underline{73.2} & 70.2 & 67.6 & 71.2 & 70.5 & 72.3 & \textbf{74.2} \\
        \rowcolor[rgb]{0.99,0.95,0.90}
        I-VG (4) & 77.2 & 84.9 & \underline{90.4} & \textbf{91.4} & 82.1 & 86.3 & 84.1 & 79.7 & 82.7 & 86.0 & 88.4 & 87.8 \\
        \rowcolor[rgb]{0.93,0.98,0.93}
        V-CLS (5) & 44.3 & 48.6 & \underline{57.0} & \textbf{57.6} & 48.0 & 52.6 & 47.3 & 45.0 & 53.0 & 53.6 & 52.2 & 53.6 \\
        \rowcolor[rgb]{0.93,0.98,0.93}
        V-QA (5) & 51.2 & 60.7 & 55.9 & 58.4 & 52.1 & 62.0 & 49.1 & 52.3 & 61.3 & \textbf{63.3} & \underline{63.2} & 62.6 \\
        \rowcolor[rgb]{0.93,0.98,0.93}
        V-RET (5) & 32.9 & 38.2 & \textbf{45.1} & \textbf{45.1} & 33.6 & 38.4 & 33.2 & 33.5 & 36.2 & 38.0 & \underline{39.5} & 38.6 \\
        \rowcolor[rgb]{0.93,0.98,0.93}
        V-MR (3) & 39.7 & 39.3 & \underline{49.4} & \textbf{49.5} & 39.2 & 41.6 & 32.1 & 46.7 & 35.4 & 39.9 & 40.4 & 38.6 \\
        \rowcolor[rgb]{0.96,0.94,0.99}
        VD-ViDoRe-V1 (10) & 72.4 & 75.7 & 79.9 & \underline{80.2} & 76.4 & \textbf{80.9} & 77.5 & 72.1 & 77.4 & 78.6 & 80.1 & \textbf{80.9} \\
        \rowcolor[rgb]{0.96,0.94,0.99}
        VD-ViDoRe-V2 (4) & 46.2 & 50.5 & 52.0 & 53.4 & 51.4 & 55.6 & 53.2 & 49.8 & 51.8 & 55.4 & \textbf{58.2} & \underline{56.1} \\
        \rowcolor[rgb]{0.96,0.94,0.99}
        VD-VisRAG (6) & 79.2 & 83.7 & 84.6 & 84.9 & 81.7 & \textbf{85.8} & 83.2 & 78.1 & 84.2 & 84.2 & \underline{85.7} & \underline{85.7} \\
        \rowcolor[rgb]{0.96,0.94,0.99}
        VD-OOD (4) & 37.2 & 37.6 & 65.7 & 67.1 & 63.9 & 66.9 & 41.1 & 57.4 & 65.2 & 66.7 & \textbf{68.1} & \underline{67.3} \\
        \midrule
        ImageNet-1K & 75.3 & 80.4 & 78.0 & 79.5 & 81.2 & 80.9 & \textbf{83.3} & 74.1 & \underline{82.0} & 78.5 & 80.1 & 81.4 \\
        N24News & 81.1 & \underline{82.3} & 44.9 & 48.3 & 80.0 & \textbf{82.7} & 78.6 & 81.1 & 81.0 & 79.7 & 81.3 & \textbf{82.7} \\
        HatefulMemes & 75.2 & \textbf{79.0} & 65.0 & 66.2 & 68.4 & \underline{76.2} & 64.0 & 75.5 & 71.4 & 71.6 & 72.5 & 75.9 \\
        VOC2007 & 80.0 & 90.8 & 78.7 & 79.5 & 90.4 & 91.0 & 86.3 & 86.1 & 91.6 & \textbf{93.3} & \textbf{93.3} & \underline{92.2} \\
        SUN397 & 79.4 & 80.3 & 75.4 & 79.2 & 80.1 & 80.6 & 77.5 & 76.9 & 78.9 & \underline{81.2} & \textbf{81.3} & 79.3 \\
        Place365 & 42.6 & \underline{46.8} & 43.9 & 43.1 & 45.3 & 45.5 & 45.7 & 42.4 & 45.7 & \underline{46.8} & \textbf{47.7} & \textbf{47.7} \\
        ImageNet-A & 50.4 & 53.9 & \underline{59.2} & 58.1 & 52.1 & 57.4 & 50.9 & 50.8 & 51.8 & 57.3 & \textbf{61.2} & 56.7 \\
        ImageNet-R & 88.7 & 90.1 & 88.5 & 88.2 & 89.9 & 89.9 & 89.7 & 87.5 & 90.8 & 90.5 & \textbf{92.4} & \underline{90.9} \\
        ObjectNet & 52.0 & 42.3 & \underline{74.8} & \textbf{75.4} & 66.1 & 72.7 & 74.1 & 61.5 & 72.5 & 70.6 & 73.5 & 71.1 \\
        Country211 & 23.4 & 25.0 & 20.0 & 19.4 & 25.5 & \underline{26.4} & \textbf{28.5} & 25.0 & 22.0 & 19.8 & 19.6 & 24.8 \\
        OK-VQA & 62.4 & 71.7 & 61.4 & 67.3 & 65.8 & \underline{74.1} & 68.4 & 60.5 & 68.4 & 69.6 & 72.3 & \textbf{74.5} \\
        A-OKVQA & 51.1 & 58.7 & 54.7 & 59.3 & 56.4 & 61.8 & 57.1 & 49.9 & 60.0 & 60.5 & \underline{62.3} & \textbf{62.7} \\
        DocVQA & 92.2 & 93.8 & 92.4 & 94.3 & 93.4 & 94.4 & 94.2 & 89.9 & 94.8 & 94.6 & \underline{94.9} & \textbf{95.6} \\
        InfographicsVQA & 67.7 & \textbf{79.2} & 76.7 & 77.5 & 63.6 & \underline{79.1} & 65.6 & 59.6 & 74.6 & 74.4 & 76.1 & 77.3 \\
        ChartQA & 64.9 & 75.1 & 80.7 & 80.9 & 61.7 & 77.4 & 57.5 & 49.8 & 81.6 & 83.9 & \textbf{85.1} & \underline{84.0} \\
        Visual7W & 54.1 & 55.2 & 52.7 & 55.3 & 54.8 & 54.9 & 54.1 & 47.6 & 54.6 & \textbf{58.5} & 56.6 & \underline{58.3} \\
        ScienceQA & 42.7 & 53.7 & 57.3 & 61.6 & 45.7 & 59.0 & 50.7 & 42.9 & 58.8 & 58.3 & \textbf{62.6} & \underline{61.9} \\
        VizWiz & 46.8 & 51.6 & 54.5 & \textbf{56.2} & 47.7 & \underline{55.3} & 55.1 & 46.5 & 53.9 & 51.4 & 53.4 & 52.5 \\
        GQA & 67.3 & 69.3 & 64.9 & 68.5 & 73.1 & \underline{73.6} & \textbf{77.0} & 69.1 & 71.4 & 70.3 & 69.3 & 72.2 \\
        TextVQA & 78.6 & 83.5 & 83.8 & 84.3 & 81.3 & \textbf{87.0} & 86.2 & 78.9 & 84.4 & 85.9 & 85.8 & \underline{86.4} \\
        VisDial & 76.6 & 80.7 & 81.5 & \textbf{84.9} & 80.7 & 82.6 & 81.2 & 72.6 & 82.8 & 82.5 & \underline{83.5} & 83.2 \\
        CIRR & 53.7 & 55.3 & 47.6 & \underline{61.2} & 57.0 & 60.0 & 59.4 & 54.6 & 55.6 & 60.8 & \textbf{61.7} & 60.1 \\
        VisualNews\_t2i & 71.7 & 76.8 & 71.9 & 73.7 & 72.4 & \textbf{79.8} & 72.8 & 71.3 & 75.4 & 71.1 & 74.8 & \underline{79.2} \\
        VisualNews\_i2t & 74.2 & 82.0 & 73.6 & 73.9 & 78.2 & \textbf{83.5} & 76.5 & 72.7 & 78.2 & 73.5 & 77.0 & \underline{82.7} \\
        MSCOCO\_t2i & 75.1 & 78.3 & \underline{79.4} & 78.9 & 76.1 & 77.8 & 75.2 & 74.1 & 76.8 & 78.7 & \textbf{79.5} & 78.8 \\
        MSCOCO\_i2t & 68.9 & 71.4 & 75.3 & \underline{76.3} & 70.0 & 72.6 & 71.1 & 69.8 & 73.0 & 75.7 & \textbf{76.5} & \underline{76.3} \\
        NIGHTS & 67.2 & 68.1 & 66.3 & 66.4 & 67.9 & 68.6 & \textbf{70.8} & 68.0 & 68.2 & 66.9 & \underline{69.1} & 68.9 \\
        WebQA & 90.0 & 90.9 & 89.3 & 90.5 & 90.7 & 90.7 & 90.4 & 89.1 & 90.8 & 90.3 & \textbf{92.2} & \underline{91.5} \\
        FashionIQ & 17.1 & 23.4 & 24.0 & \textbf{31.9} & 19.8 & 26.0 & 26.3 & 20.3 & 22.7 & 25.5 & \textbf{31.9} & \underline{30.8} \\
        Wiki-SS-NQ & 62.0 & 72.5 & 68.9 & 69.6 & 68.9 & \textbf{76.3} & 64.2 & 68.6 & 72.4 & 70.1 & 71.9 & \underline{76.2} \\
        OVEN & 66.9 & \textbf{71.4} & 61.4 & 60.7 & 67.5 & 68.6 & 67.6 & 68.4 & 68.7 & 62.9 & 63.1 & \underline{69.7} \\
        EDIS & 88.0 & \underline{92.0} & 84.5 & 87.4 & 88.9 & 91.6 & 87.0 & 81.8 & 89.5 & 88.4 & 86.1 & \textbf{92.7} \\
        MSCOCO & 69.5 & 72.7 & \underline{92.9} & \textbf{93.6} & 69.0 & 72.0 & 67.7 & 66.9 & 73.3 & 78.1 & 80.5 & 76.1 \\
        RefCOCO & 83.3 & 91.4 & \underline{94.9} & \textbf{95.9} & 88.9 & 91.7 & 91.4 & 86.5 & 89.2 & 93.0 & 94.8 & 93.8 \\
        RefCOCO-Matching & 84.4 & 91.1 & 85.8 & 88.0 & 90.1 & \underline{93.7} & \textbf{95.0} & 88.4 & 89.6 & 88.1 & 89.6 & 93.5 \\
        Visual7W-Pointing & 71.5 & 84.2 & \underline{88.0} & 87.9 & 80.5 & 87.9 & 82.5 & 74.9 & 78.7 & 84.6 & \textbf{88.8} & 87.8 \\
        \bottomrule
    \end{tabular}
\end{table}
\clearpage

\begin{table}[!htbp]
    \centering
    \caption{Detailed MMEB-V2 results on the 18 video datasets (Hit@1)
    and 24 visual-document datasets (NDCG@5).
    Model abbreviations follow Table~\ref{tab:appendix-mmeb-image}.
    Best and second-best scores are bolded and underlined, respectively.}
    \label{tab:appendix-mmeb-video-visdoc}
    \vspace{\baselineskip}
    \scriptsize
    \setlength{\tabcolsep}{1pt}
    \begin{tabular}{@{}p{0.35\textwidth}*{7}{>{\centering\arraybackslash}p{\dimexpr0.053333\textwidth-2pt\relax}}>{\centering\arraybackslash}p{\dimexpr0.063333\textwidth-2pt\relax}*{4}{>{\centering\arraybackslash}p{\dimexpr0.053333\textwidth-2pt\relax}}@{}}
        \toprule
        \multirow{2}{*}{Dataset / task group}
            & \multicolumn{2}{c}{UME-R1}
            & \multicolumn{2}{c}{Embed-RL}
            & \multicolumn{2}{c}{RIME}
            & \multicolumn{1}{@{}c@{}}{TTE$_s$}
            & \multicolumn{1}{@{}c@{}}{PLUME}
            & \multicolumn{4}{c}{ReWAM} \\
        \cmidrule(lr){2-3}\cmidrule(lr){4-5}
        \cmidrule(lr){6-7}\cmidrule(lr){8-8}\cmidrule(lr){9-9}
        \cmidrule(l){10-13}
        & 2B & 7B & 2B & 4B & 2B & 7B & 2B & 2B
        & Q2-2B & Q3-2B & Q3-4B & Q2-7B \\
        \midrule
        K700 & 35.8 & 42.8 & \underline{55.8} & \textbf{56.8} & 47.7 & 55.0 & 49.6 & 42.2 & 52.1 & 50.7 & 53.7 & 55.3 \\
        SmthSmthV2 & 44.1 & 50.4 & \underline{56.7} & \textbf{59.5} & 48.5 & 55.1 & 50.4 & 44.8 & 53.5 & 53.0 & 55.2 & 55.1 \\
        HMDB51 & 54.4 & 58.3 & 56.7 & \textbf{60.1} & 56.7 & \underline{58.9} & 52.5 & 51.2 & 58.5 & 55.9 & 54.3 & 55.9 \\
        UCF101 & 67.2 & 70.0 & \textbf{79.3} & \underline{78.5} & 66.4 & 68.5 & 58.3 & 66.5 & 72.0 & 76.3 & 75.3 & 74.7 \\
        Breakfast & 20.1 & 21.5 & \textbf{36.7} & \underline{33.0} & 20.6 & 25.4 & 25.4 & 20.1 & 28.9 & 32.3 & 22.3 & 26.8 \\
        MVBench & 49.9 & 58.2 & 50.8 & 55.9 & 49.1 & \textbf{59.9} & 48.5 & 47.4 & 56.5 & 58.8 & \underline{59.0} & \underline{59.0} \\
        Video-MME & 41.7 & 47.3 & 47.1 & \textbf{50.5} & 41.6 & 49.6 & 45.8 & 40.0 & 48.9 & \underline{50.4} & \underline{50.4} & 49.5 \\
        NExTQA & 59.9 & 69.6 & 53.9 & 58.2 & 58.9 & 69.7 & 53.8 & 57.3 & 70.3 & \textbf{71.1} & 69.2 & \underline{71.0} \\
        EgoSchema & 45.4 & 52.4 & 53.0 & 52.8 & 40.4 & 55.6 & 36.4 & 47.8 & 55.2 & \underline{59.6} & \textbf{62.1} & 58.2 \\
        ActivityNetQA & 57.8 & \underline{76.0} & 74.8 & 74.4 & 70.5 & 75.1 & 60.8 & 69.2 & 75.4 & \textbf{76.5} & 75.5 & 75.4 \\
        DiDeMo & 32.4 & 40.0 & \underline{45.3} & \textbf{46.8} & 33.8 & 38.4 & 33.5 & 32.7 & 35.4 & 37.5 & 38.6 & 39.0 \\
        MSR-VTT & 34.3 & 38.9 & \underline{45.7} & \textbf{46.2} & 36.4 & 41.5 & 34.8 & 36.2 & 40.4 & 42.2 & 42.6 & 42.6 \\
        MSVD & 55.4 & 60.8 & \textbf{67.2} & \underline{65.8} & 56.6 & 59.4 & 56.5 & 56.1 & 57.8 & 60.6 & 63.8 & 59.1 \\
        VATEX & 29.9 & 32.6 & \textbf{43.6} & \underline{43.4} & 28.1 & 32.7 & 25.6 & 28.2 & 32.3 & 34.0 & 35.4 & 35.8 \\
        YouCook2 & 12.7 & 18.5 & \textbf{23.5} & \underline{23.3} & 13.3 & 19.9 & 15.8 & 14.5 & 15.1 & 15.9 & 17.1 & 16.3 \\
        QVHighlight & 57.5 & 54.9 & \underline{70.7} & \textbf{73.6} & 55.1 & 56.9 & 38.9 & 57.1 & 43.8 & 54.7 & 57.7 & 50.3 \\
        Charades-STA & 20.4 & 21.9 & \textbf{26.4} & \underline{25.0} & 19.4 & 21.6 & 19.5 & 19.4 & 18.1 & 19.3 & 20.0 & 19.0 \\
        MomentSeeker & 41.2 & 41.1 & \underline{50.9} & 49.9 & 43.2 & 46.2 & 37.7 & \textbf{63.5} & 44.2 & 45.6 & 43.4 & 46.6 \\
        \midrule
        ViDoRe\_arxivqa & 73.9 & 73.6 & 86.1 & \textbf{88.7} & 82.1 & 84.1 & 80.7 & 72.6 & 85.1 & 86.2 & \underline{87.6} & 86.8 \\
        ViDoRe\_docvqa & 37.9 & 41.1 & 45.7 & 47.5 & 45.4 & 46.9 & 44.5 & 36.2 & 44.7 & \textbf{48.7} & \underline{48.6} & 47.9 \\
        ViDoRe\_infovqa & 76.2 & 80.8 & 86.8 & \underline{86.9} & 81.2 & 85.4 & 84.8 & 79.0 & 83.4 & 85.0 & \underline{86.9} & \textbf{88.0} \\
        ViDoRe\_tabfquad & 86.1 & 90.2 & 94.5 & \underline{94.7} & 85.5 & \textbf{95.3} & 88.4 & 88.8 & 93.2 & 92.9 & 94.0 & 93.5 \\
        ViDoRe\_tatdqa & 40.6 & 46.7 & 54.6 & \underline{54.8} & 48.9 & 52.6 & 50.4 & 36.6 & 50.1 & 50.8 & 52.4 & \textbf{55.5} \\
        ViDoRe\_shiftproject & 66.8 & 65.0 & 70.7 & 69.0 & 69.9 & \underline{73.7} & 65.2 & 64.8 & 69.5 & 63.9 & 68.0 & \textbf{74.2} \\
        ViDoRe\_artificial\_intelligence & 85.9 & 89.5 & \underline{94.0} & 91.6 & 89.8 & \textbf{96.1} & 91.9 & 83.8 & 88.5 & 92.2 & 93.1 & 92.0 \\
        ViDoRe\_energy & 83.3 & 85.7 & 86.7 & 88.1 & 86.2 & 88.5 & \underline{88.7} & 82.6 & 84.7 & 86.4 & \textbf{88.8} & 87.4 \\
        ViDoRe\_government\_reports & 82.6 & 89.8 & 89.0 & \underline{90.7} & 86.4 & \textbf{91.7} & 86.9 & 83.2 & 84.6 & 89.2 & 89.7 & 89.4 \\
        ViDoRe\_healthcare\_industry & 90.8 & \underline{94.3} & 91.1 & 90.4 & 88.3 & \textbf{94.9} & 92.8 & 91.1 & 90.2 & 90.9 & 91.9 & 94.0 \\
        ViDoRe\_esg\_reports\_human\_labeled\_v2 & 50.2 & 50.4 & 56.9 & 59.8 & 56.5 & \underline{60.2} & 59.0 & 52.1 & 54.5 & 59.0 & 59.5 & \textbf{60.5} \\
        ViDoRe\_biomedical\_lectures\_v2\_multilingual & 46.2 & 50.7 & 51.0 & 50.1 & 47.6 & 51.5 & 52.0 & 48.2 & 52.5 & \underline{55.3} & \textbf{58.3} & 52.6 \\
        ViDoRe\_economics\_reports\_v2\_multilingual & 45.7 & 57.8 & 53.0 & 53.9 & 47.0 & 59.2 & 49.8 & 49.6 & 49.3 & 55.1 & \underline{59.5} & \textbf{60.2} \\
        ViDoRe\_esg\_reports\_v2\_multilingual & 42.7 & 43.2 & 46.9 & 49.7 & \underline{54.3} & 51.6 & 52.1 & 49.0 & 51.0 & 52.0 & \textbf{55.3} & 51.0 \\
        VisRAG\_ArxivQA & 74.3 & 80.5 & 84.9 & \textbf{86.9} & 79.4 & 84.0 & 78.5 & 71.6 & 83.0 & 83.0 & \underline{85.3} & 84.9 \\
        VisRAG\_ChartQA & 86.0 & 85.0 & \underline{88.3} & \textbf{88.5} & 83.3 & 85.4 & 84.4 & 80.8 & 88.2 & 86.1 & 86.6 & 85.0 \\
        VisRAG\_MP-DocVQA & 75.6 & 83.4 & 79.1 & 79.3 & 82.3 & \textbf{87.4} & 79.2 & 74.9 & 81.8 & 82.1 & 84.4 & \underline{86.6} \\
        VisRAG\_SlideVQA & 87.1 & 91.5 & 92.3 & 92.6 & 91.0 & \textbf{94.1} & 92.3 & 88.9 & 91.4 & 93.2 & \underline{94.0} & \underline{94.0} \\
        VisRAG\_InfoVQA & 84.4 & 89.2 & \underline{90.0} & 89.6 & 82.4 & \textbf{91.8} & 87.2 & 85.7 & 87.7 & 88.5 & 89.5 & 89.5 \\
        VisRAG\_PlotQA & 68.0 & 72.7 & 73.0 & 72.4 & 71.6 & 72.1 & \textbf{77.5} & 66.2 & 73.3 & 72.4 & 74.2 & \underline{74.4} \\
        ViDoSeek-page & 21.2 & 21.3 & 82.0 & \underline{84.4} & 80.7 & \textbf{85.6} & 22.6 & 80.6 & 81.8 & 81.5 & 82.2 & 81.1 \\
        ViDoSeek-doc & 75.9 & 75.3 & 82.6 & 82.4 & 79.7 & 80.9 & 82.0 & 76.9 & 80.5 & 82.3 & \textbf{83.1} & \underline{82.7} \\
        MMLongBench-page & 11.9 & 12.3 & 47.7 & 51.0 & 48.3 & 52.4 & 12.9 & 39.9 & 49.5 & 52.5 & \textbf{55.2} & \underline{53.7} \\
        MMLongBench-doc & 39.7 & 41.3 & 50.3 & 50.7 & 46.9 & 48.8 & 47.0 & 32.0 & 49.1 & 50.4 & \textbf{51.9} & \underline{51.8} \\
        \bottomrule
    \end{tabular}
\end{table}

\clearpage

\subsection{Qualitative Case Studies}
\label{app:qualitative-cases}

\paragraph{Reasoning quality and adaptive inference.}
Figures~\ref{fig:case-text}--\ref{fig:case-video-vqa} illustrate the reasoning
quality of the RASD-trained reasoner. Its full CoTs focus on salient
entities, actions, document statistics, and spatial relations, distilling
multimodal inputs into informative descriptions relevant to retrieval.
RAI further shortens these CoTs by omitting later restatements and secondary
elaborations while retaining the core matching cues. Query--positive cosine
similarities remain close to the full-CoT values, suggesting that the
shortened reasoning largely preserves semantic correspondence.

\par\smallskip
\noindent\begin{minipage}{\linewidth}
    \centering
    \includegraphics[width=\linewidth]{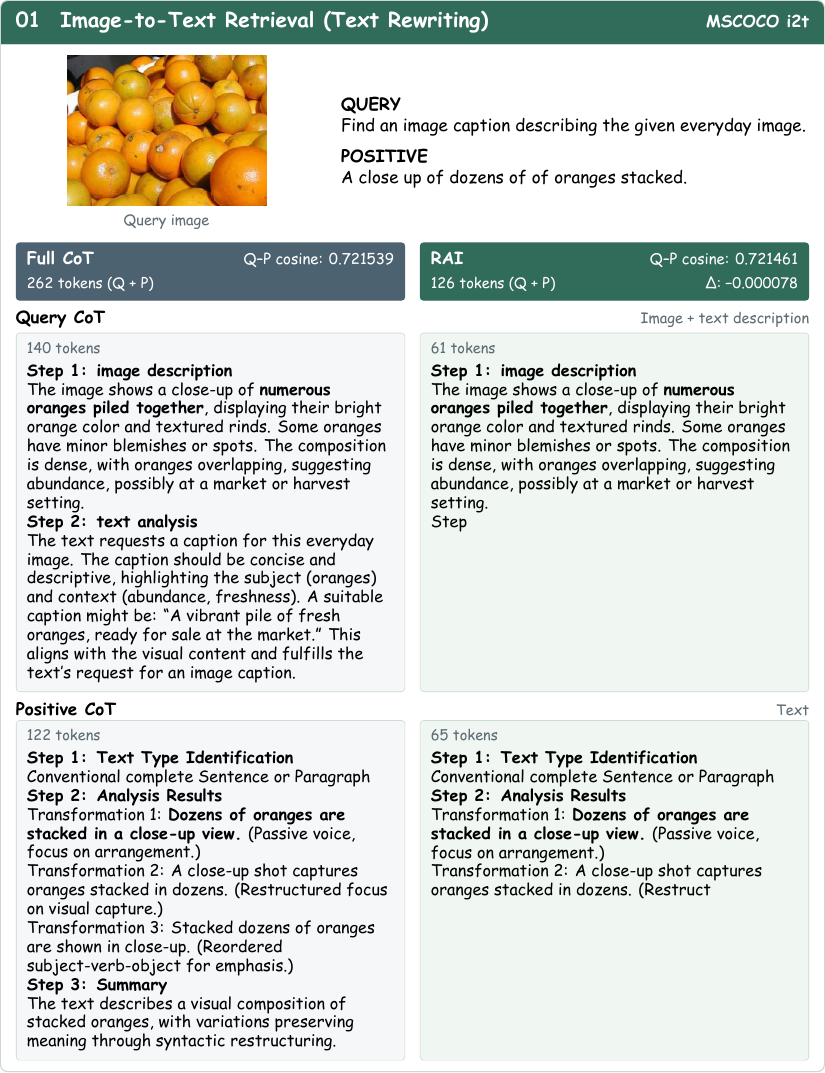}
    \captionof{figure}{An image-to-text retrieval case on MSCOCO.
    The RAI panels show the generated CoTs up to their actual stopping positions.}
    \label{fig:case-text}
    \label{fig:appendix-cases}
\end{minipage}

\clearpage
\noindent\begin{minipage}{\linewidth}
    \centering
    \includegraphics[width=\linewidth]{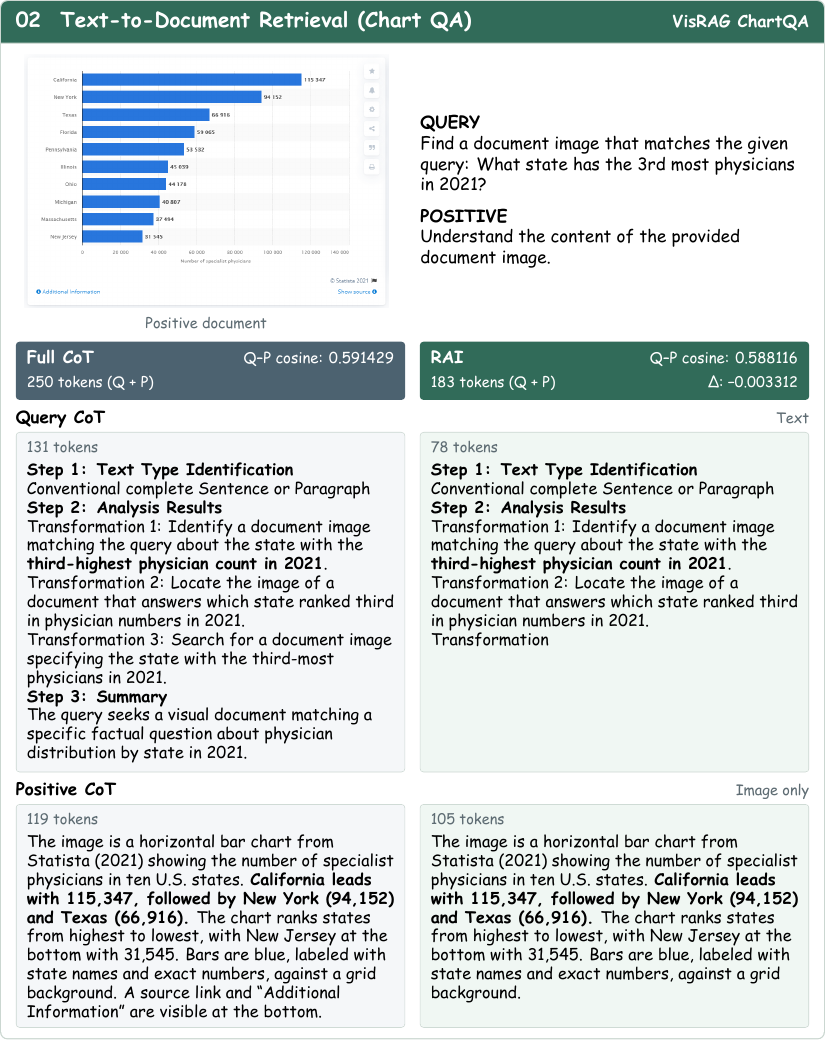}
    \captionof{figure}{A text-to-document retrieval case on VisRAG-ChartQA.
    The RAI panels show the generated CoTs up to their actual stopping positions.}
    \label{fig:case-image}
\end{minipage}

\clearpage
\noindent\begin{minipage}{\linewidth}
    \centering
    \includegraphics[width=\linewidth]{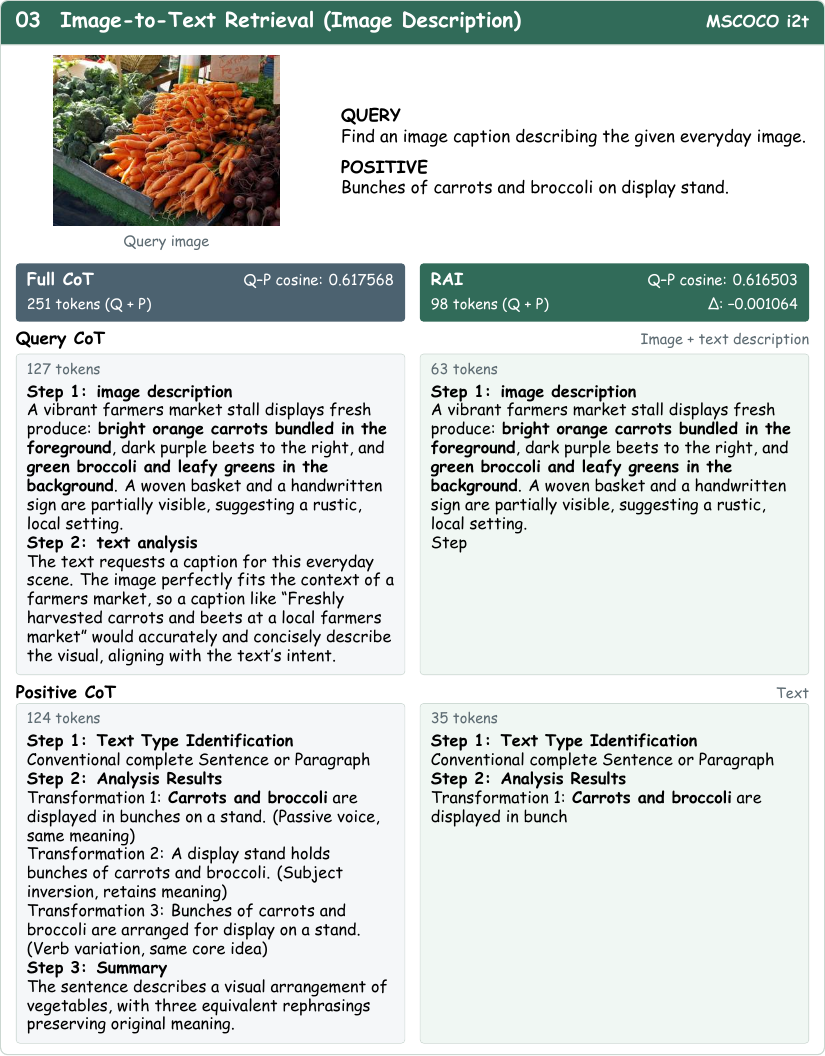}
    \captionof{figure}{An image-to-text retrieval case on MSCOCO.
    The RAI panels show the generated CoTs up to their actual stopping positions.}
    \label{fig:case-image-text}
\end{minipage}

\clearpage
\noindent\begin{minipage}{\linewidth}
    \centering
    \includegraphics[width=\linewidth]{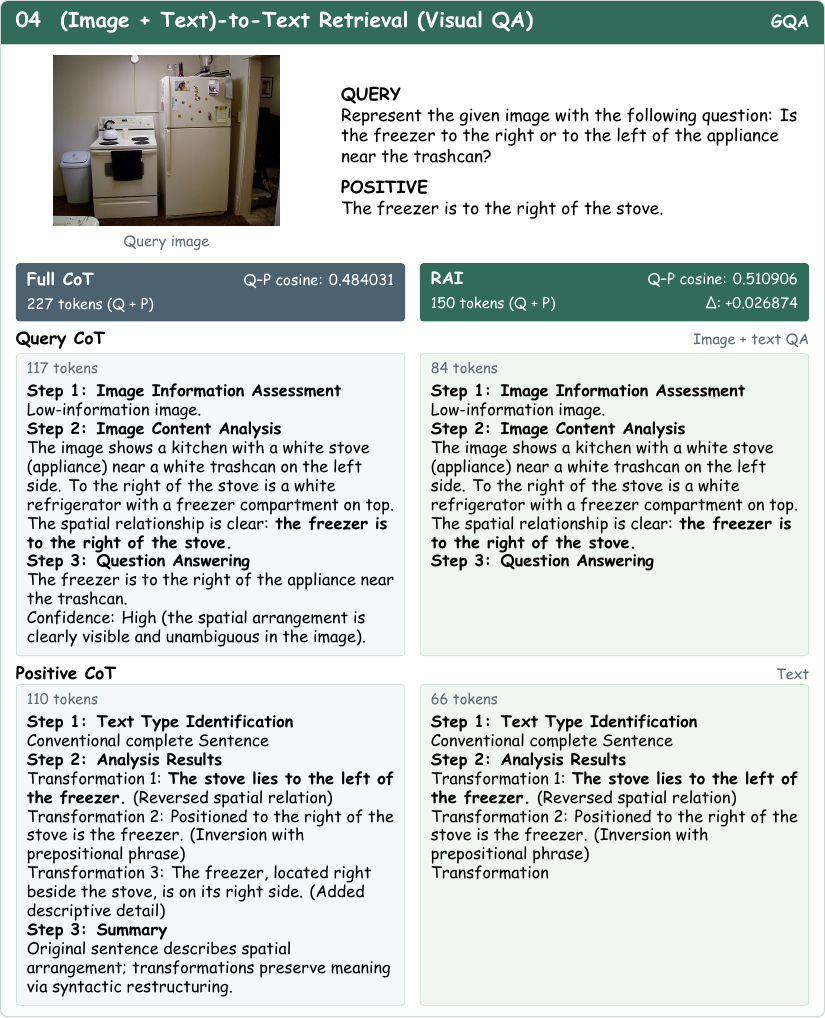}
    \captionof{figure}{A visual question answering case on GQA.
    The RAI panels show the generated CoTs up to their actual stopping positions.}
    \label{fig:case-image-vqa}
\end{minipage}

\clearpage
\noindent\begin{minipage}{\linewidth}
    \centering
    \includegraphics[width=\linewidth]{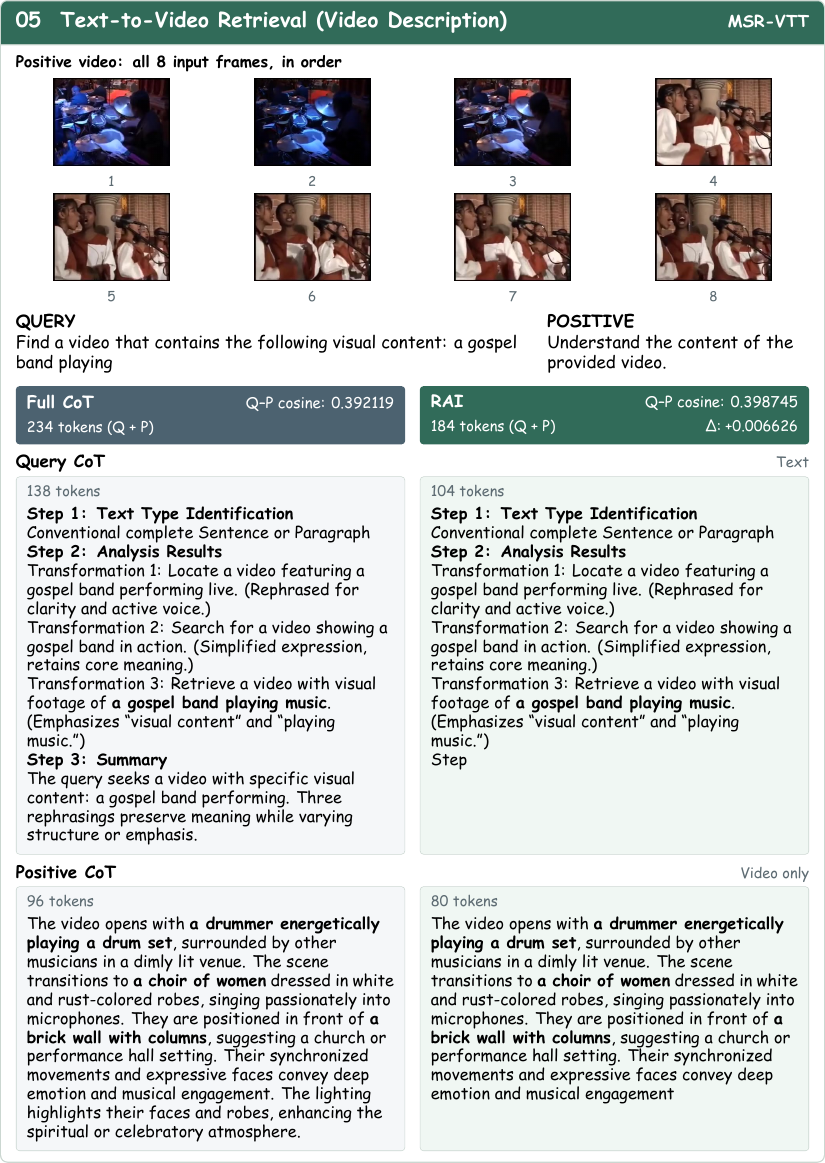}
    \captionof{figure}{A text-to-video retrieval case on MSR-VTT.
    The RAI panels show the generated CoTs up to their actual stopping positions.}
    \label{fig:case-video}
\end{minipage}

\clearpage
\noindent\begin{minipage}{\linewidth}
    \centering
    \includegraphics[width=\linewidth]{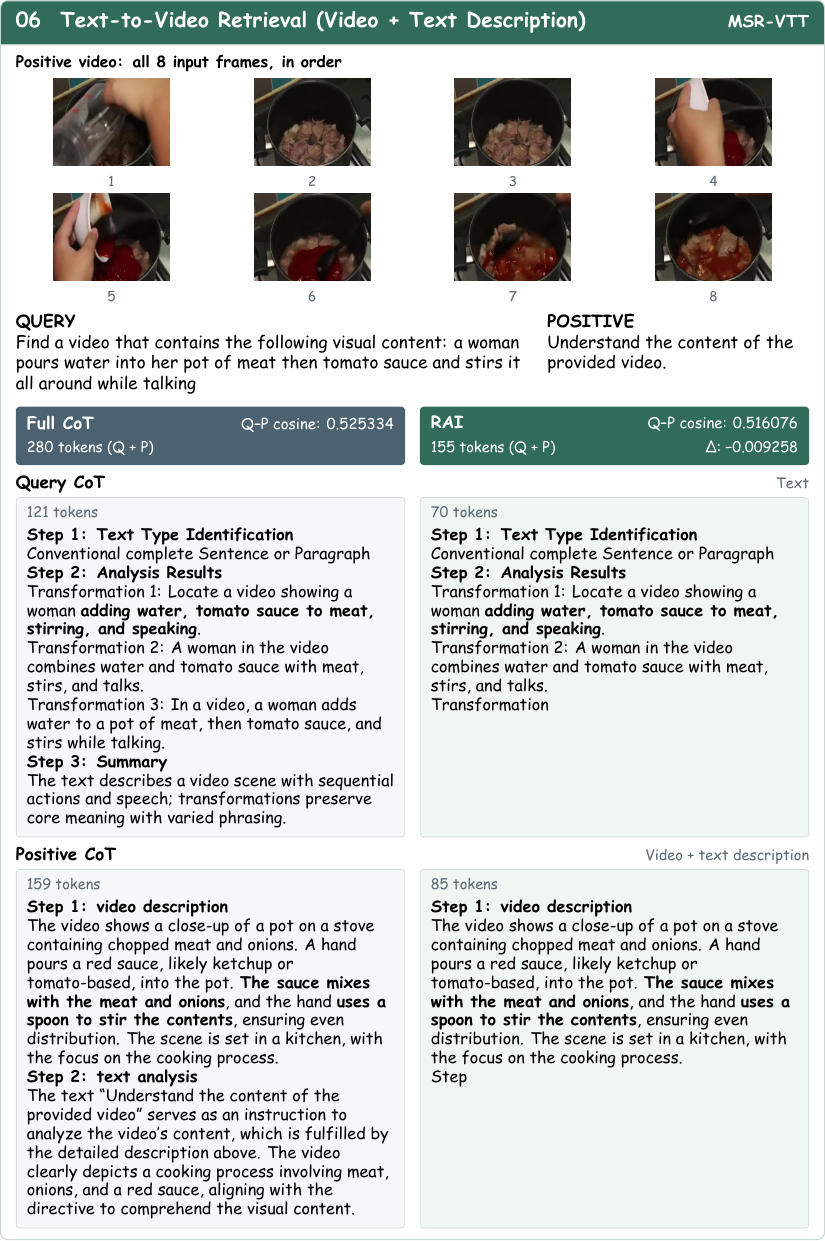}
    \captionof{figure}{A text-to-video retrieval case on MSR-VTT.
    The RAI panels show the generated CoTs up to their actual stopping positions.}
    \label{fig:case-video-text}
\end{minipage}

\clearpage
\noindent\begin{minipage}{\linewidth}
    \centering
    \includegraphics[width=\linewidth]{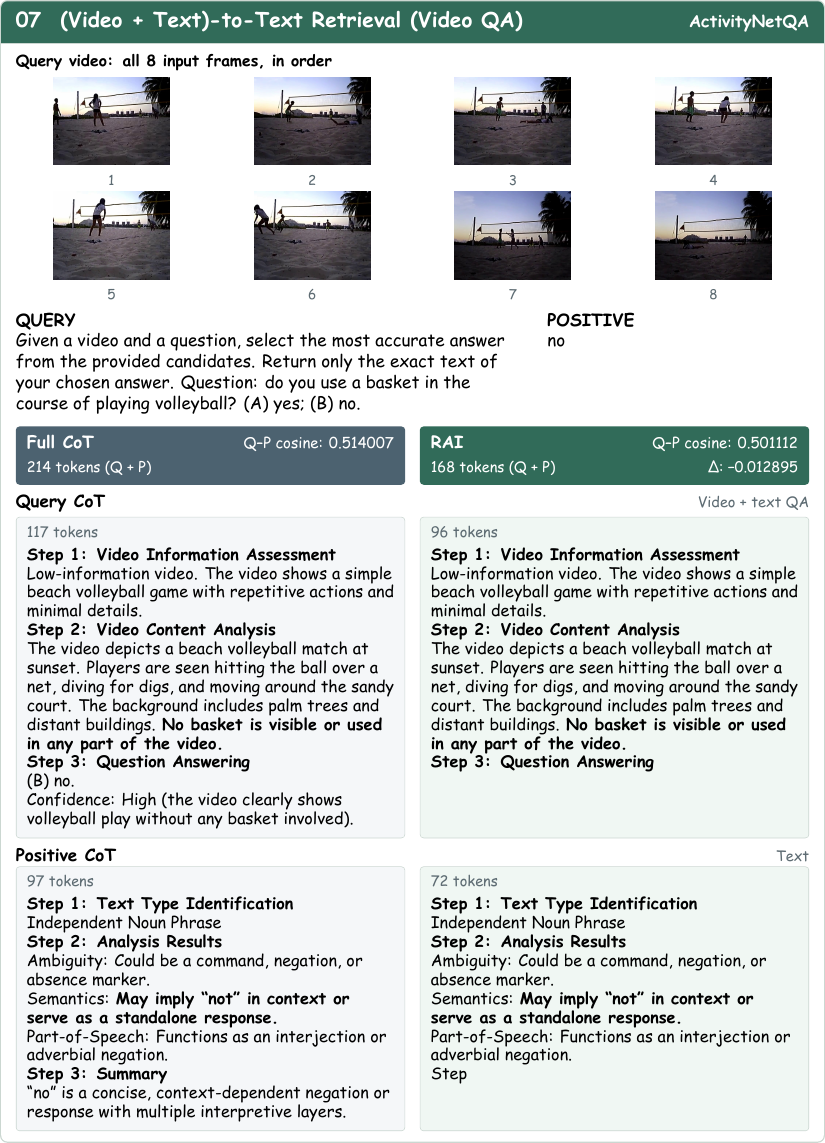}
    \captionof{figure}{A video question answering case on ActivityNetQA.
    The RAI panels show the generated CoTs up to their actual stopping positions.}
    \label{fig:case-video-vqa}
\end{minipage}

\clearpage
\paragraph{CEPI-conditioned teacher probabilities.}
Figures~\ref{fig:case-cepi-tennis}--\ref{fig:case-cepi-tiered} visualize
the teacher's token probabilities across four cases. We score the saved
CoTs offline under fixed prefixes using the same reference reasoner
conditioned on CEPI. Superscripts report the CEPI-induced change in token
log-probability, $\Delta$; $\bar\Delta$ denotes the mean over tokens in an
annotated span. Positive and negative values indicate increased and decreased
model support, respectively. In the VisDial examples, high-probability regions
cover several retrieval-relevant visual details, including clothing and
scene attributes, while some low-probability regions correspond to generic
evaluations and task restatements. Their image-description steps also have
higher mean teacher probabilities than their task-analysis steps.
The VisualNews and NIGHTS cases illustrate factual conflicts: a custodian
helmet is mislabeled as a ``peaked cap,'' and seating on a flat floor is
described as a ``tiered layout.'' CEPI explicitly identifies these
discrepancies, and the flagged tokens ``peaked'' and ``layout'' receive low
teacher probabilities, consistent with the conflict explanations.
CEPI increases support for several grounded descriptions while substantially
reducing support for the conflicting headwear and seating-layout claims.

\par\smallskip
\noindent\begin{minipage}{\linewidth}
    \centering
    \includegraphics[width=\linewidth]{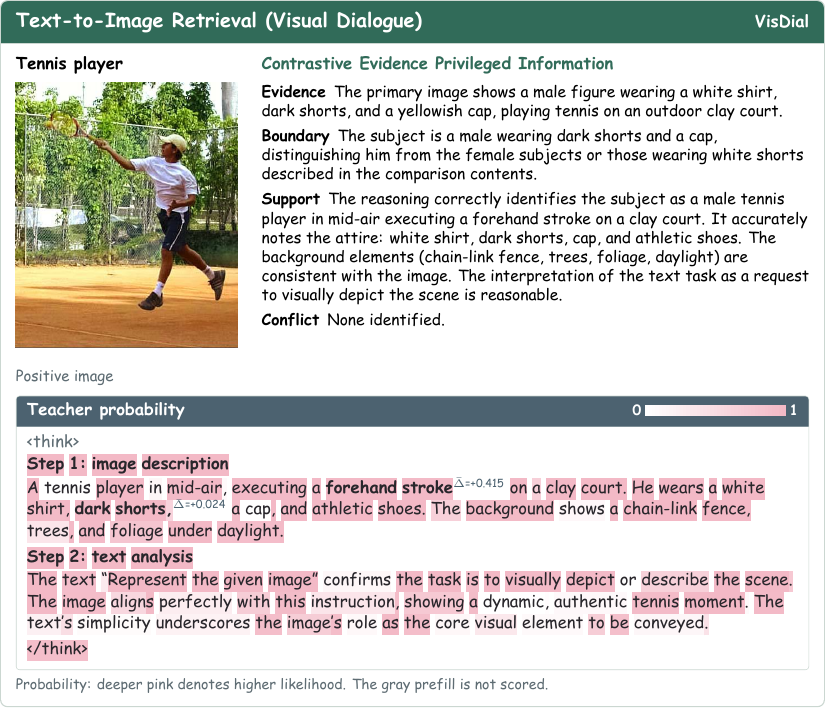}
    \captionof{figure}{CEPI-conditioned teacher probabilities on VisDial (tennis).}
    \label{fig:case-cepi-tennis}
\end{minipage}

\clearpage
\noindent\begin{minipage}{\linewidth}
    \centering
    \includegraphics[width=\linewidth]{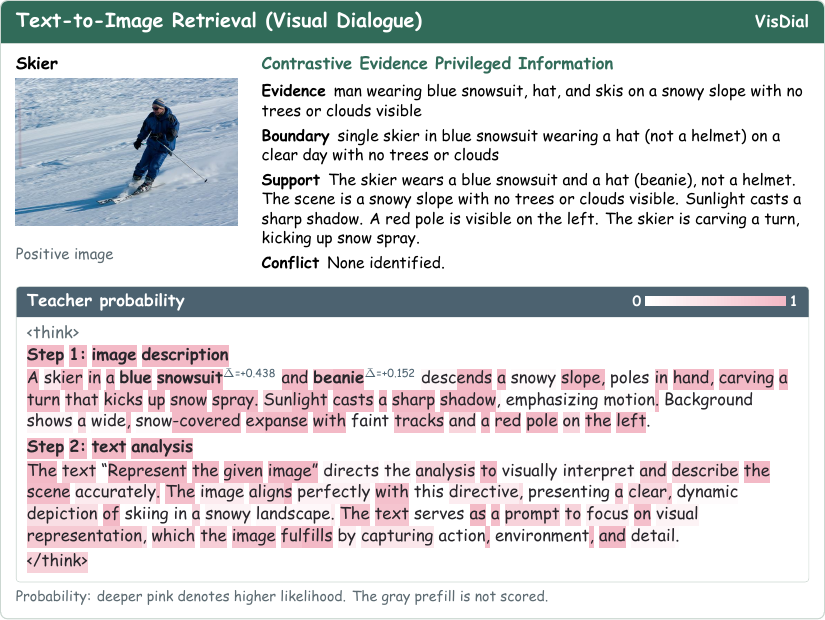}
    \captionof{figure}{CEPI-conditioned teacher probabilities on VisDial (skiing).}
    \label{fig:case-cepi-skier}
\end{minipage}

\clearpage
\noindent\begin{minipage}{\linewidth}
    \centering
    \includegraphics[width=\linewidth]{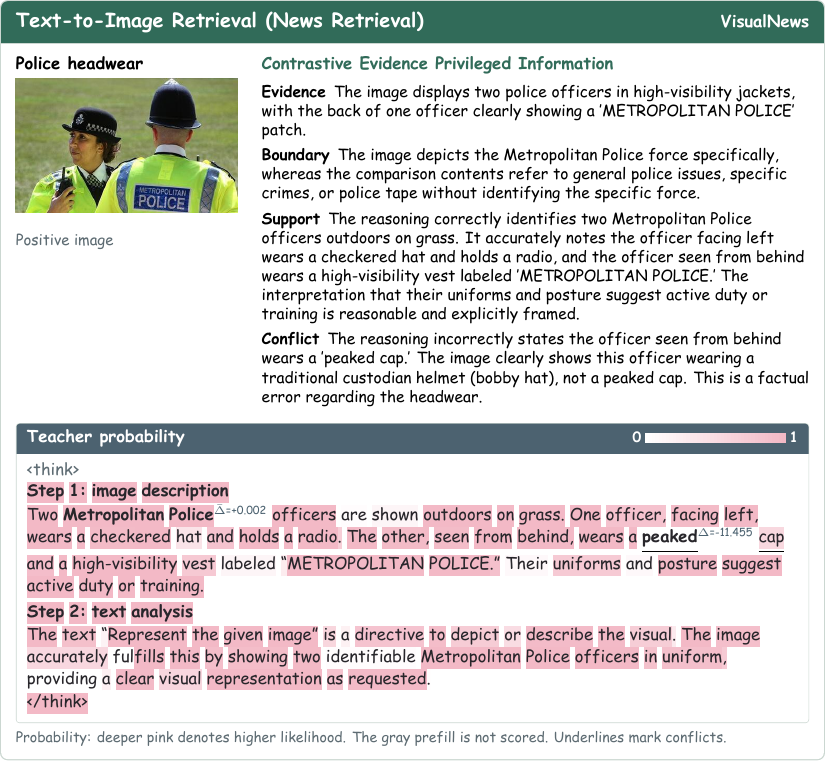}
    \captionof{figure}{A headwear conflict case on VisualNews.}
    \label{fig:case-cepi-helmet}
\end{minipage}

\clearpage
\noindent\begin{minipage}{\linewidth}
    \centering
    \includegraphics[width=\linewidth]{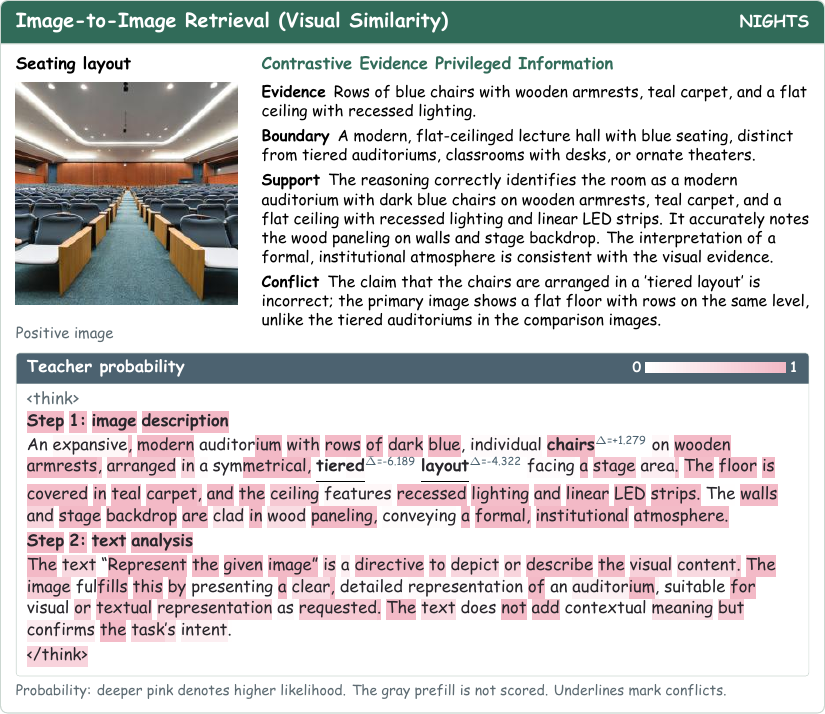}
    \captionof{figure}{A seating layout conflict case on NIGHTS.}
    \label{fig:case-cepi-tiered}
\end{minipage}

\clearpage
\section{CEPI Prompt Templates}
\label{app:analyzer-prompts}

The condensed prompts below describe how the analyzer in Figure~\ref{fig2}(b)
extracts Evidence and Boundary, then verifies each sampled CoT to identify
Support and Conflict.

\begin{center}
    \includegraphics[width=\linewidth]{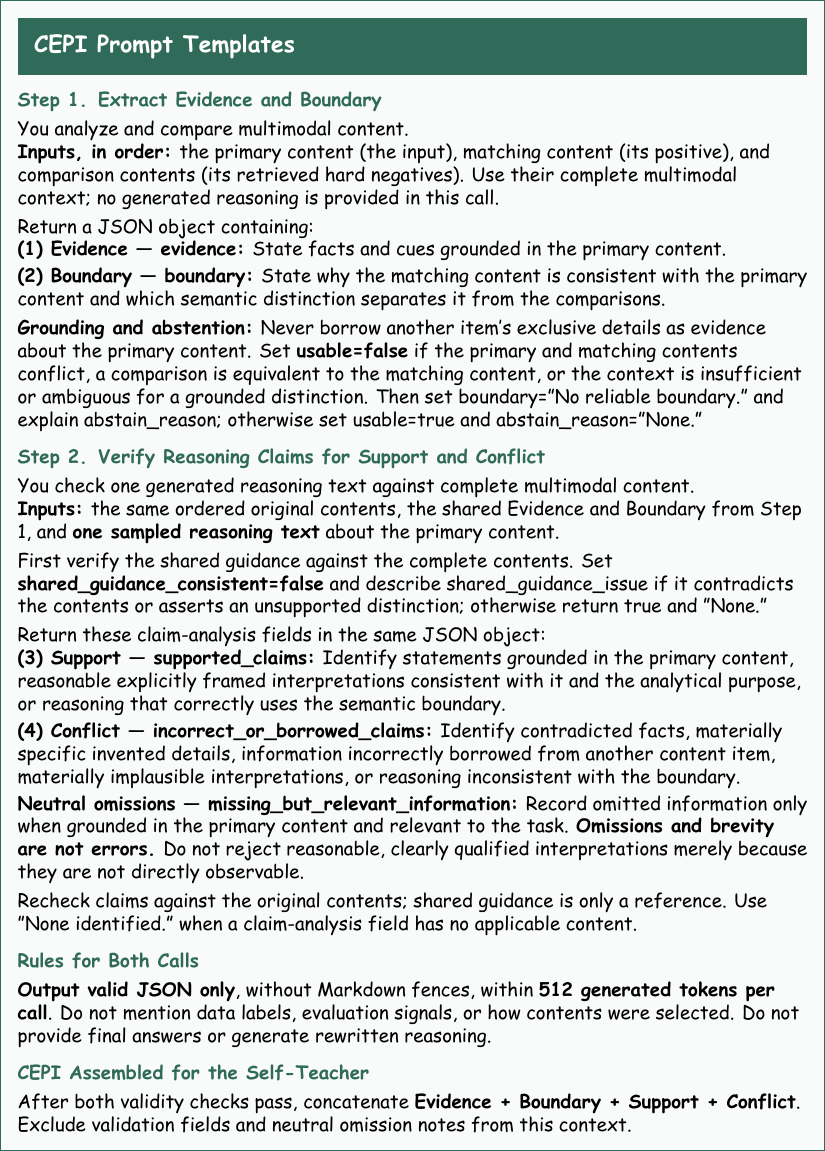}
\end{center}

\clearpage
\section{More Discussions}
\label{app:discussion}

As shown in Tables~\ref{tab:mmeb-v2-results} and~\ref{tab:mrmr-results},
ReWAM achieves strong overall performance on MMEB-V2 and MRMR.
For a fixed corpus and encoding configuration, each target's CoT and embedding
can be computed once offline, with the embedding stored in a retrieval index.
Subsequent requests then require CoT generation and embedding extraction only
for the incoming query, followed by a search over the cached target vectors.
RAI accelerates both initial corpus encoding and online query processing.
However, query-side CoT generation still adds latency, and CEPI construction
requires additional analyzer calls during training.
Future work could distill the analyzer's evidence extraction and claim
verification capabilities into the reasoner to generate its own training-time
privileged guidance, reducing external analyzer calls. Another future direction
is to learn a unified policy that balances retrieval gains against decoding costs
to decide whether to generate CoT, when to stop, and how many tokens to draft
at each step.


\stopcontents[appendix]

\end{document}